\documentclass{article}

\usepackage[accepted]{icml2020}
\usepackage{graphicx}

\usepackage{algorithm} % algorithm table
\usepackage[noend]{algpseudocode} % algblockdefx

\usepackage{amssymb}
\usepackage{hyperref}
\usepackage{enumitem} % itemize
\usepackage{amsmath} % text
\usepackage[tbtags]{mathtools}

\usepackage{multirow}
\usepackage{subcaption} % subfigure, subcaptionbox
\usepackage{makecell}

\usepackage{amsthm}
\usepackage{booktabs}
\usepackage{balance}

\usepackage[T1]{fontenc}
\usepackage{CJKutf8}
\usepackage[english]{babel}
\usepackage[all]{nowidow}

\theoremstyle{definition}
\newtheorem{definition}{Definition}
\newtheorem{theorem}{Theorem}
\newtheorem{lemma}{Lemma}

\newtheorem{remark}{Remark}

\algblockdefx[If]{If}{EndIf}[1]{\textbf{if} #1 \textbf{then}}{end if}
\algblockdefx[Foreach]{Foreach}{EndForeach}[1]{\textbf{for each} #1 \textbf{do}}{end foreach}
\algblockdefx[ParForeach]{ParForeach}{EndParForeach}[1]{\textbf{for each} #1 \textbf{in parallel do}}{end foreach}
\algblockdefx[While]{While}{EndWhile}[1]{\textbf{while} #1 \textbf{do}}{end while}
\algblockdefx[Until]{Until}{EndUntil}[1]{\textbf{until} #1 \textbf{do}}{end until}
\algblockdefx[Function]{Function}{EndFunction}[2]{\textbf{function} #1(#2)}{end function}
\algblockdefx[Procedure]{Procedure}{EndProcedure}[2]{\textbf{procedure} #1(#2)}{end procedure}

\makeatletter
\ifthenelse{\equal{\ALG@noend}{t}}%
  {\algtext*{EndIf}}
  {}%
\ifthenelse{\equal{\ALG@noend}{t}}%
  {\algtext*{EndForeach}}
  {}%
\ifthenelse{\equal{\ALG@noend}{t}}%
  {\algtext*{EndParForeach}}
  {}%
\ifthenelse{\equal{\ALG@noend}{t}}%
  {\algtext*{EndWhile}}
  {}%
\ifthenelse{\equal{\ALG@noend}{t}}%
  {\algtext*{EndUntil}}
  {}%
\ifthenelse{\equal{\ALG@noend}{t}}%
  {\algtext*{EndFunction}}
  {}%
\ifthenelse{\equal{\ALG@noend}{t}}%
  {\algtext*{EndProcedure}}
  {}%
\makeatother

\let\oldReturn\Return
\renewcommand{\Return}{\State\oldReturn}

\usepackage{array}

\newcommand{\g}{\nabla}

\DeclarePairedDelimiter\p{\lparen}{\rparen}
\DeclarePairedDelimiter\ang{\langle}{\rangle} % <angle brackets>
\DeclarePairedDelimiter\abs{\lvert}{\rvert}   % |absolute value|
\DeclarePairedDelimiter\norm{\lVert}{\rVert}  % ||norm||
\DeclarePairedDelimiter\bkt{[}{]}             % [brackets]
\DeclarePairedDelimiter\set{\{}{\}}           % {braces}
\DeclarePairedDelimiter\ceil{\lceil}{\rceil}
\DeclarePairedDelimiter\floor{\lfloor}{\rfloor}

\makeatletter
\let\oldp\p \def\p{\@ifstar{\oldp}{\oldp*}}
\let\oldang\ang \def\ang{\@ifstar{\oldang}{\oldang*}}
\let\oldabs\abs \def\abs{\@ifstar{\oldabs}{\oldabs*}}
\let\oldnorm\norm \def\norm{\@ifstar{\oldnorm}{\oldnorm*}}
\let\oldbkt\bkt \def\bkt{\@ifstar{\oldbkt}{\oldbkt*}}
\let\oldset\set \def\set{\@ifstar{\oldset}{\oldset*}}
\let\oldceil\ceil \def\ceil{\@ifstar{\oldceil}{\oldceil*}}
\let\oldfloor\floor \def\floor{\@ifstar{\oldfloor}{\oldfloor*}}
\makeatother

\newcommand{\normal}[1]{\normalfont{\text{#1}}}
\renewcommand{\v}{\normalfont{\textbf{v}}}
\newcommand{\w}{\normalfont{\textbf{w}}}
\newcommand{\x}{\normalfont{\textbf{x}}}
\newcommand{\z}{\normalfont{\textbf{z}}}

\def\gD{\mathcal{D}}
\def\gN{\mathcal{N}}
\def\gK{\mathcal{K}}

\DeclareMathOperator*{\argmin}{arg\,min}

\newcommand{\assign}{\leftarrow}

\usepackage{tikz}
\newcommand*\circled[1]{\tikz[baseline=(char.base)]{\node[shape=circle,draw,inner sep=0.2pt] (char) {\normal{#1}};}}

\usepackage{upgreek} % uptau

\icmltitlerunning{Accurate and Fast Federated Learning via IID and Communication-Aware Grouping}

\begin{document}

\twocolumn[
\icmltitle{Accurate and Fast Federated Learning \\ via IID and Communication-Aware Grouping}



\icmlsetsymbol{equal}{*}

\begin{icmlauthorlist}
\icmlauthor{Jin-woo Lee}{to}
\icmlauthor{Jaehoon Oh}{to}
\icmlauthor{Yooju Shin}{to}
\icmlauthor{Jae-Gil Lee}{to}
\icmlauthor{Se-Young Yoon}{to}
\end{icmlauthorlist}

\icmlaffiliation{to}{Korea Advanced Institute of Science and Technology}

\icmlcorrespondingauthor{Jin-woo Lee}{jinwoo.lee@kaist.ac.kr}
% \icmlcorrespondingauthor{Eee Pppp}{ep@eden.co.uk}

\icmlkeywords{Machine Learning, ICML}

\vskip 0.3in
]



\printAffiliationsAndNotice{\icmlEqualContribution} % otherwise use the standard text.

\begin{abstract}
Federated learning has emerged as a new paradigm of collaborative machine learning; however, it has also faced several challenges such as non-independent and identically distributed\,(IID) data and high communication cost. To this end, we propose a novel framework of \emph{IID and communication-aware group federated learning} that simultaneously maximizes both accuracy and communication speed by grouping nodes based on data distributions and physical locations of the nodes. Furthermore, we provide a formal convergence analysis and an efficient optimization algorithm called \textbf{\textit{FedAvg-IC}}. Experimental results show that, compared with the state-of-the-art algorithms, FedAvg-IC improved the test accuracy by up to $22.2\%$ and simultaneously reduced the communication time to as small as $12\%$.
\end{abstract}

\section{Introduction}
\emph{Federated learning}\,\citep{FSVRG,FedAvg} enables mobile devices to collaboratively learn a shared model while keeping all training data on the devices, thus avoiding transferring data to the cloud or central server. In this framework, a \emph{local} model is updated using the data on each device, and all local updates are periodically aggregated to the \emph{global} model; then, each local model is synchronized with the global model. Federated learning is attracting more attention, as indicated by the recent release of TensorFlow Federated\,(TFF) in March 2019\,\citep{TFF}. One of the main reasons for this recent boom in federated learning is that it does not compromise user privacy.
However, there are several challenges despite federated learning's growing popularity. \citet{FedAvg} pointed out that federated learning has three unique properties: non-independent and identically distributed\,(IID), unbalanced, and massively-distributed. In this study, we tackle the challenges for the non-IID and massively-distributed properties as follows:
\begin{itemize}[noitemsep,leftmargin=10pt,nosep]
\label{Challenges}
\item
    \textbf{Non-IID Challenge}: Because each mobile device typically stores the data generated by a particular user, each local data distribution does \emph{not} represent the global population distribution. This non-IID property definitely hinders the convergence of federated learning and degrades prediction accuracy.% On the other hand, this issue may not be critical in traditional distributed machine learning because all training data can be placed on the central server.
\item
    \textbf{Limited Communication Challenge}: Because several thousands of devices typically participate in federated learning, the training process is massively distributed, thus causing a huge burden on the backbone\,(wireless) network\,\citep{park2018wireless}.% Even worse, these devices may spread all over the place. On the other hand, in traditional distributed processing, the number of computing servers is not very large, and they are usually located in a single data center.
\end{itemize}

To the best of our knowledge, no existing work has addressed both of the above challenges simultaneously. However, there have been active studies on each challenge. Notably, \citet{HierSGD} proposed a group-based learning algorithm, where the nodes\,(i.e., devices) are grouped into node groups; the local models are first aggregated to a \emph{group} model, and the group models are then aggregated to the global model. While the group-based learning relieves the non-IID issue, it may cause high communication overheads especially if far-away nodes belong to the same node group. \citet{AdaptiveFL} proposed a resource-constrained optimization algorithm to optimize the number of communication rounds but did not address the non-IID issue. In contrast, \citet{IIDSharing} proposed a data sharing strategy that distributes a small subset of global data to all nodes for resolving the non-IIDness, but the additional global communication cost is not seriously considered and it somewhat violates the philosophy of federated learning.

\begin{figure}[t!]
    \centering
    \includegraphics[width=0.85\columnwidth]{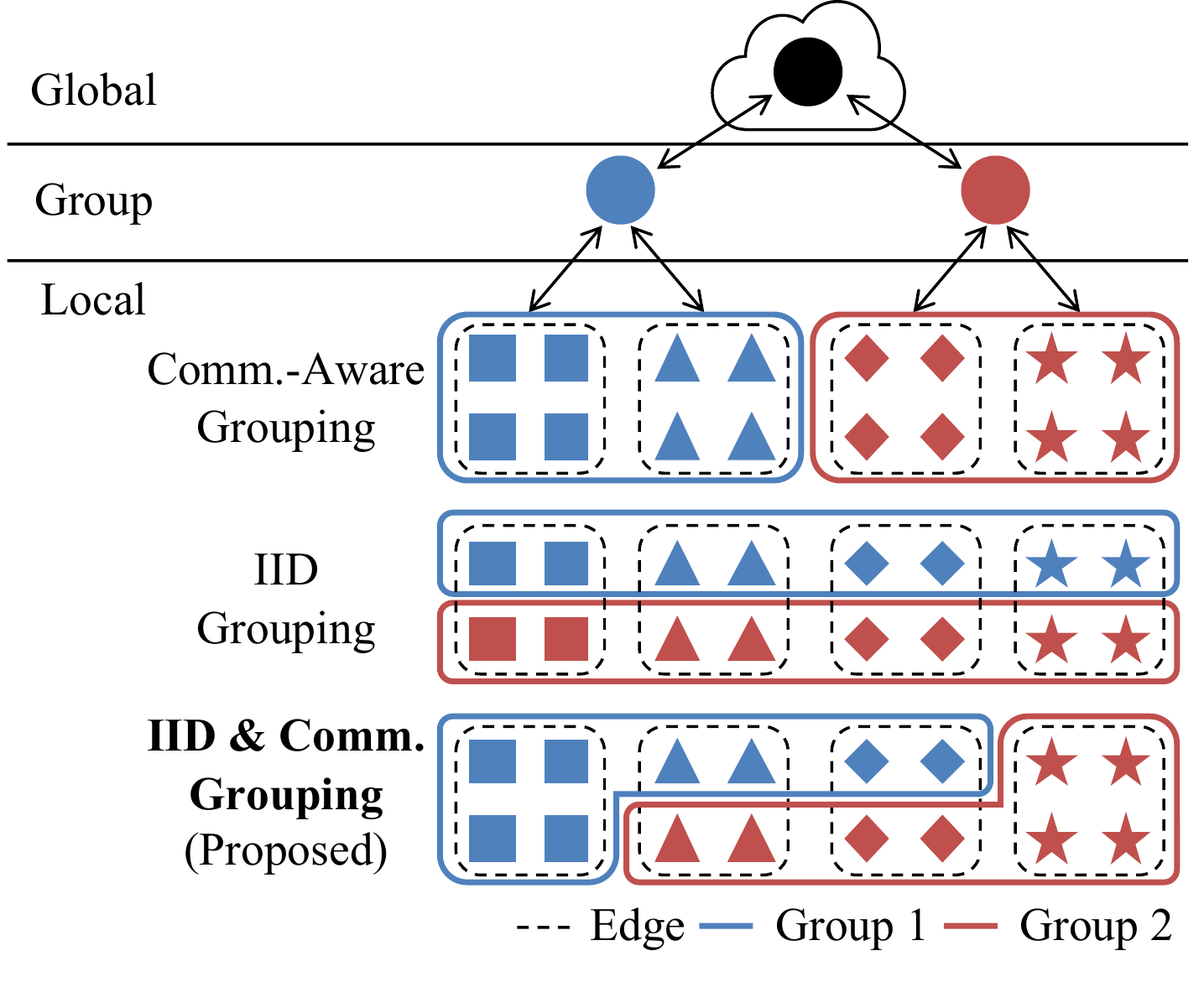}
    \vspace*{-0.4cm}
    \caption{Concept of IID and communication-aware group federated learning.}
    \label{fig:Overview}
    \vspace*{-0.8cm}
\end{figure}

In this paper, we propose a novel framework of \emph{IID and communication-aware group federated learning} to address both challenges. Here, nodes are grouped by the \emph{IID and communication-aware grouping} principle to make the data distribution of each group closer to the global IID data distribution and to reduce node-to-group communication simultaneously. \autoref{fig:Overview}, where the data distributions are distinguished by their different shapes, illustrates the proposed framework as well as two simple alternatives. Communication-aware grouping concentrates on the limited communication challenge at the cost of accuracy, and IID grouping concentrates on the non-IID challenge at the cost of efficiency. On the other hand, our proposed framework aims at presenting a \emph{hybrid} of the two extreme cases.
Overall, the key contributions are summarized as follows:
%\,(circle, rectangle, triangle, diamond, and star) <- deleted by Jae-Gil
% First, we group nodes based on the \emph{IID grouping} principle to make the data distribution of a group closer to the global IID distribution.
% Next, based on the \emph{communication-aware grouping} principle, we aim to reduce the overall communication duration by striking the balance between node-to-group communication\,(arrow between an edge and a group) and group-to-global communication\,(arrow between a group and the cloud). For instance, in \autoref{fig:Overview}, we group nodes with considering both challenges,
% where the data distribution difference between local models is distinguished by their different shapes\,(circle, rectangle, diamond, and triangle). Overall, the key contributions are summarized as follows:

\begin{itemize}[noitemsep,leftmargin=10pt,nosep]
\label{Contributions}
\item
    \textbf{Problem Formulation}\,(Section \ref{sec:ProblemSetting}): We formulate the problem as a bi-objective optimization that determines node groups by considering the difference in data distribution for the local-to-group and group-to-global levels as well as the communication delay based on the physical locations of nodes.
    % To address both challenges, the node grouping is determined by considering the difference in data distribution for the local-to-group and group-to-global levels as well as the communication delay based on the physical locations of nodes.
\item
    \textbf{Convergence Analysis}\,(Section \ref{sec:TheoreticalAnalysis}):
    We formally derive the convergence bound of group federated learning.
    % and then use this bound to approximately solve our problem.
    As per our analysis, the optimal node grouping is achieved when the difference in data distribution for the group-to-global level and the group communication delay are simultaneously minimized.% and the number of group aggregations is one of the main factors affecting the communication duration.
\item
    \textbf{Optimization Algorithm}\,(Section \ref{sec:Optimization}): We design a control algorithm, called \textit{FedAvg-IC}, to find the near-optimal node grouping that minimizes both IID and communication costs.% In addition, we suggest a practical implementation technique to reduce actual amount of communication data on the network.
\item
    \textbf{High Performance}\,(Section \ref{sec:Evaluation}): We empirically compared \textit{FedAvg-IC} with three federated learning algorithms on four benchmark datasets. \textit{FedAvg-IC} reached a higher accuracy by up to $22.2\%$ and simultaneously reduced communication time to as small as $12\%$, compared with the three selected algorithms.
% \item
%     \textbf{Compatibility or Extensibility}\,(Appendix~\ref{sec:Supplementary}): We experiment with our algorithm in combinations between node numbers and edges under various networks (e.g., networks type, processing speed, and link speed). In addition, common communication-efficient methods such as sampling and quantization are compatible with our algorithm.
\end{itemize}
\vspace*{-0.2cm}
\section{Preliminaries and Related Work}
In this section, we first briefly describe \emph{federated learning} and then survey relevant studies that handle either \emph{non-IIDness} or \emph{limited communication}.

\vspace*{-0.2cm}
\subsection{Basics of Federated Learning}
The objective of \emph{federated learning} is to find an approximate solution of Eq.~\eqref{eq:FLP}\,\citep{FedAvg}. Here, $F(\w)$ is the loss of predictions with a model $ \w $ over the set of all data examples $ \gD \triangleq \cup_{ i \in \gN }{ \gD_i } $ across all nodes, where $ \gN $ is the set of node indices, $ F_i\p{\w} \triangleq \sum_{ j \in \gD_i }{ \frac{ 1 }{ \abs{\gD_i} } F_{ij}\p{\w} } $ is the loss of predictions with $ \w $ over the set of data examples $ \gD_i $ on the $ i $-th node, and $ F_{ij}\p{\w} \triangleq l\p{ \w, \x_{ij}, y_{ij} } $ is the loss of a prediction with $ \w $ on the $ j $-th data example $ \p{ \x_{ij}, y_{ij} } $ on the $ i $-th node.
\vspace*{-0.2cm}
\begin{equation}
\label{eq:FLP}
    \w^* \triangleq \argmin_{ \w \in \mathbb{R}^d }{ F\p{\w} } ~~ \normal{where} ~~ F\p{\w} \triangleq \sum_{ i \in \gN }{ \frac{ \abs{\gD_i} }{ \abs{\gD} } F_i\p{\w} }
\vspace*{-0.1cm}
\end{equation}
\emph{Federated Averaging\,(FedAvg)}\,\citep{FedAvg}, which is the canonical algorithm for federated learning in Eq.~\eqref{eq:FLP}, involves \emph{local update}, which learns a \emph{local model} $ \w_i $ at the $ i $-th node by performing gradient descent steps, and \emph{global aggregation}, which learns the \emph{global model} $ \w $ by aggregating all $ \w_i $ and synchronizes $ \w_i $ with $ \w $ every $ \uptau $ steps, as shown in Eq.~\eqref{eq:w_FedAvg}.
\vspace*{-0.2cm}
\begin{equation}
\label{eq:w_FedAvg}
\resizebox{.9\columnwidth}{!}{$
\begin{multlined}
    \hspace{-2pt} \w_i\p{t} \hspace{-1pt} \triangleq \hspace{-2pt} \begin{dcases}
        \w_i\p{t \hspace{-1pt} - \hspace{-1pt} 1} \hspace{-1pt} - \hspace{-1pt} \eta\g F_i\p{\w_i\p{t \hspace{-1pt} - \hspace{-1pt} 1}} \hspace{-6pt} & \normal{if} ~ t \normal{ mod } \uptau \ne 0 \\
        \w\p{t} & \normal{if} ~ t \normal{ mod } \uptau = 0
    \end{dcases} \\
    \normal{where} \hspace{5pt} \w\p{t} \triangleq \sum\limits_{ i \in N }{ \frac{ \abs{\gD_i} }{ \abs{\gD} }\bkt{ \w_i\p{t-1} - \eta\g F_i\p{\w_i\p{t-1}} } }
\end{multlined}
$}
\end{equation}

\vspace*{-0.5cm}
\subsection{Related Work on the Non-IID Challenge}
%Except using multiple models as in federated multi-task learning \citep{FedMultiTask} and federated meta learning \citep{FedMetaLearning}, 
%\emph{group federated learning} and \emph{global-information sharing} are mainly involved in learning a global model for resolving the non-IID challenge.

\subsubsection{Group Federated Learning}
To reduce the learning divergence between $ \w_i $ and $ \w $ in Eq.~\eqref{eq:w_FedAvg}, \citet{HierSGD} proposed a group-based architecture of allowing multiple intermediate aggregations before a global aggregation. Formally speaking, the set of all node indices $ \gN $ is partitioned into sets of node indices for $ \abs{\gK} $ \emph{node groups} $ \{\gN^1, \gN^2, \cdots , \gN^{\abs{\gK}} \} $, i.e., $ \cup_{ k \in \gK }{ \gN^k } = \gN $ and $ \forall k \ne l $, $ \gN^k \cap \gN^l = \emptyset $. Additionally, let $\gD^k$ be the set of data examples on the $k$-th node group and $\gD^k_i$ be its subset of $\gD^k$ on the $i$-th node. Then, the loss function of Eq.~\eqref{eq:FLP} is extended to that of Eq.~\eqref{eq:loss_GroupFL} by considering the node groups.
\vspace*{-0.2cm}
\begin{equation}
\label{eq:loss_GroupFL}
\resizebox{.9\columnwidth}{!}{$
    \hspace{-2pt} F\p{\w} \hspace{-1pt} \triangleq \hspace{-2pt} \sum\limits_{ k \in \gK }{ \frac{ \abs{\gD^k} }{ \abs{\gD} } F^k\p{\w} } ~ , ~ F^k\p{\w} \hspace{-1pt} \triangleq \hspace{-2pt} \sum\limits_{ i \in \gN^k }{ \frac{ \abs{\gD^k_i} }{ \abs{\gD^k} } F^k_i\p{\w} }
$}
\vspace*{-0.3cm}
\end{equation}

Group federated learning was implemented as \emph{hierarchical local SGD}\,\citep{HierSGD}, and it learns the \emph{group model} $ \w^k $ by aggregating all $ \w_i^k $ and synchronizes $ \w_i^k $ with $ \w^k $ every $ \uptau_1 $ steps, which can be expressed as Eq.~\eqref{eq:w_GroupFL}.

\vspace*{-0.4cm}
\begin{equation}
\label{eq:w_GroupFL}
\resizebox{\columnwidth}{!}{$
\begin{multlined}
    \hspace{-2pt} \w^k_i\p{t} \hspace{-1pt} \triangleq \hspace{-2pt} \begin{dcases}
        \w^k_i\p{t \hspace{-1pt} - \hspace{-1pt} 1} \hspace{-1pt} - \hspace{-1pt} \eta\g F^k_i\p{\w^k_i\p{t \hspace{-1pt} - \hspace{-1pt} 1}} \hspace{-6pt} & \normal{if} ~ t \normal{ mod } \uptau_1 \ne 0 \\
        \w^k\p{t} & \normal{if} ~ \begin{aligned}[t]
            & t \normal{ mod } \uptau_1 = 0, \\
            & t \normal{ mod } \uptau_1\uptau_2 \ne 0
        \end{aligned} \\
        \w\p{t} & \normal{if} ~ t \normal{ mod } \uptau_1\uptau_2 = 0
    \end{dcases} \\
    \mbox{where} \hspace{5pt} \w^k\p{t} \hspace{-1pt} \triangleq \hspace{-2pt} \sum\limits_{ i \in \gN^k }{ \hspace{-1pt} \frac{ \abs{\gD^k_i} }{ \abs{\gD^k} }\bkt{ \w^k_i\p{t \hspace{-1pt} - \hspace{-1pt} 1} \hspace{-1pt} - \hspace{-1pt} \eta\g F^k_i\p{\w^k_i\p{t \hspace{-1pt} - \hspace{-1pt} 1}} } } \\
    \mbox{and} \hspace{5pt} \w\p{t} \hspace{-1pt} \triangleq \hspace{-2pt} \sum\limits_{ k \in \gK }{ \sum\limits_{ i \in \gN^k }{ \frac{ \hspace{-1pt} \abs{\gD^k_i} }{ \abs{\gD} }\bkt{ \w^k_i\p{t \hspace{-1pt} - \hspace{-1pt} 1} \hspace{-1pt} - \hspace{-1pt} \eta\g F^k_i\p{\w^k_i\p{t \hspace{-1pt} - \hspace{-1pt} 1}} } } }
\end{multlined}
$}
\end{equation}

\vspace*{-0.4cm}
\subsubsection{Global-Information Sharing}
Sharing global information is effective in mitigating the non-IIDness of a local node. The most common approach is to share a subset of global IID data samples to make the local data distribution closer to the population data distribution\,\citep{IIDSharing,HybridFL}. FSVRG\,\citep{FSVRG} shares a subset of global data features to scale up the feature-related parameters of a local optimizer. FAug\,\citep{FAug} shares a generative model that can produce an augmented IID dataset.

\subsection{Related Work on the Communication Challenge} % by Jae-Gil to reduce one line
% \subsection{Related Work on the Limited Communication Challenge}

% There are two main directions for resolving the limited communication challenge: \emph{communication-aware learning} and \emph{communication overhead reduction}.

\subsubsection{Communication-Aware Learning}
AdaptiveFL\,\citep{AdaptiveFL} extends FedAvg to adaptively optimize the number of global aggregations by considering resource consumption such as communication. FedCS\,\citep{FedCS} minimizes the overall communication delay for a set of sampled learners by considering a round-trip time constraint.
%Hybrid-FL\,\citep{HybridFL} extends the communication constraint of FedCS by additionally considering the channel condition.
HierFAVG\,\cite{HierFAVG}, which is the state-of-the-art approach for group federated learning, groups nodes by network edges to facilitate communication between the nodes in proximity.
%The multi-objective evolutionary federated learning algorithm proposed by \citet{zhu2019multi} considers model error and communication cost to optimize the architecture of the neural network models.
Similarly, we define a novel optimization problem that considers both IID and communication costs for maximizing accuracy and efficiency of federated learning, as shown in Section \ref{sec:ProblemSetting}.

\subsubsection{Communication Overhead Reduction}
Reducing communication overheads in federated learning usually leads to saving both communication and computation resources. The overheads include the number of participating nodes and the amount of communication data. The participating nodes can be sampled by following a certain probability distribution\,\citep{FedAvg,SampledFL,FedProx}, but this approach is beyond the scope of this paper. Meanwhile, communication data size can be reduced by using a quantization or compression technique\,\cite{konecny2016federatedLearning,sattler2019robust} or by placing intermediate parameter servers in a network topology\,\cite{bonawitz2019towards}. We also attempt to reduce communication data size in Section \ref{sec:Optimization}.

\begin{table}[h!]
\caption{Summary of the notation.}
\centering
\small
\begin{tabular}{>{\centering}p{.17\columnwidth}>{\arraybackslash}p{.73\columnwidth}} \toprule
    \multicolumn{1}{c}{\textbf{Notation}} & \multicolumn{1}{c}{\textbf{Description}} \\ \midrule
    $\w^k_i$      & \emph{Local} model of $i$-th node in $k$-th group \\
    $\w^k$        & \emph{Group} model of $k$-th group \\
    $\w\p{T}$     & \emph{Global} model after $ T $ steps \\
    % $\w^{*}$      & \emph{Optimal} global model \\
    % $\z$          & \emph{Group} membership for all nodes \\
    $\uptau_1$    & \# of \emph{local} updates per group aggregation \\
    $\uptau_2$    & \# of \emph{group} aggregations per global aggregation \\
    % $d_{group}$   & Delay for a \emph{group} aggregation \\
    % $d_{global}$  & Delay for a \emph{global} aggregation \\
    $\bkt{r}$     & \emph{G\underline{r}oup} interval, i.e., $ \bkt{ \p{r-1}\uptau_1, r\uptau_1 } $ \\
    $\bkt{l}$     & \emph{G\underline{l}obal} interval, i.e., $ \bkt{ \p{l-1}\uptau_1\uptau_2, l\uptau_1\uptau_2 } $ \\
    $\delta$      & \emph{Local-to-group} divergence \\
    $\Delta$      & \emph{Group-to-global} divergence \\ \bottomrule
\end{tabular}
\label{table:Notation}
\vspace*{-0.4cm}
\end{table}
\section{IID and Communication-Aware Group Federated Learning}
\label{sec:ProblemSetting}
Our primary goal is to train a global model that simultaneously minimizes the global loss in Eq.~\eqref{eq:loss_GroupFL} and the total communication delay by considering the aforementioned challenges, which can be formulated as Eq.~\eqref{eq:Problem}.
% Thus, \emph{IID and communication-aware group federated learning} is a bi-objective optimization problem,
\vspace*{-0.1cm}
\begin{equation}
\label{eq:Problem}
    \min_{ \uptau_1, \uptau_2, \abs{\gK}, \z } \set{ F\p{\w\p{T}}, \p{ d_{group}\p{\uptau_2-1} + d_{global} } }
\end{equation}
\vspace*{-0.4cm}
\begin{itemize}[noitemsep,leftmargin=10pt,nosep]
\item
    The \emph{IID objective} is defined as the minimization of global loss after $T$ steps, and the \emph{communication objective} is defined as the minimization of total communication delay, where $d_{group}$ and $d_{global}$ represent the communication delay\,(e.g., in seconds) spent for a single iteration of group and global aggregations, respectively. $d_{group}$ and $d_{global}$ can be easily estimated from a given network topology\,(e.g., by using hop counts\,\citep{Vahdat00epidemicrouting}). $\uptau_2-1$ implies that a global aggregation takes over a group aggregation every $ \uptau_2 $ steps.
    % $\uptau_2-1$\,(instead of $\uptau_2$) is multiplied to $d_{group}$ because a group aggregation is replaced by a global aggregation every $ \uptau_2 $ steps as in Eq.~\eqref{eq:w_GroupFL}.
\item
    The optimization parameters are the learning steps $\uptau_1$ and $\uptau_2$, the number of node groups $ \abs{\gK} $, and the group membership $ \z \triangleq \p{ z_i | \bkt{ \forall i \in \gN, \exists k \in \gK }\p{ z_i = k } } $.
\end{itemize}

\section{Theoretical Analysis}
\label{sec:TheoreticalAnalysis}
In this section, we provide a theoretical analysis of the IID and communication-aware group federated learning. Based on an assumption and definitions in Section \ref{sec:Assumption}, we analyze the convergence of group federated learning in Section \ref{sec:Convergence} and draw notable remarks for the main problem in Section \ref{sec:Interpretation}. \autoref{table:Notation} summarizes the notation used in this paper.

% \begin{table}[t!]
% \caption{Summary of the notation.}
% \centering
% \small
% \begin{tabular}{>{\centering}p{.17\columnwidth}>{\arraybackslash}p{.73\columnwidth}} \toprule
%     \multicolumn{1}{c}{\textbf{Notation}} & \multicolumn{1}{c}{\textbf{Description}} \\ \midrule
%     $\w^k_i$      & \emph{Local} model of $i$-th node in $k$-th group \\
%     $\w^k$        & \emph{Group} model of $k$-th group \\
%     $\w\p{T}$     & \emph{Global} model after $ T $ steps \\
%     % $\w^{*}$      & \emph{Optimal} global model \\
%     % $\z$          & \emph{Group} membership for all nodes \\
%     $\uptau_1$    & \# of \emph{local} updates per group aggregation \\
%     $\uptau_2$    & \# of \emph{group} aggregations per global aggregation \\
%     % $d_{group}$   & Delay for a \emph{group} aggregation \\
%     % $d_{global}$  & Delay for a \emph{global} aggregation \\
%     $\bkt{r}$     & \emph{G\underline{r}oup} interval, i.e., $ \bkt{ \p{r-1}\uptau_1, r\uptau_1 } $ \\
%     $\bkt{l}$     & \emph{G\underline{l}obal} interval, i.e., $ \bkt{ \p{l-1}\uptau_1\uptau_2, l\uptau_1\uptau_2 } $ \\
%     $\delta$      & \emph{Local-to-group} divergence \\
%     $\Delta$      & \emph{Group-to-global} divergence \\ \bottomrule
% \end{tabular}
% \label{table:Notation}
% \vspace*{-0.4cm}
% \end{table}

\vspace*{-0.2cm}
\subsection{Assumption and Definitions}
\label{sec:Assumption}
We make the following assumption for the loss function $ F^k_i $, as in many other relevant studies\,\citep{HierFAVG,AdaptiveFL}.
For every $i$ and $k$, \circled{1} $ F^k_i $ is convex\footnote{We will empirically show that a non-convex function works well in Section \ref{sec:Evaluation}.}; \circled{2} $ F^k_i $ is $ \rho $-Lipschitz, i.e., $ \norm{ F^k_i(\w) - F^k_i(\w') } \le \rho\norm{ \w - \w' } $ for any $\w$ and $ \w' $; and  \circled{3} $ F^k_i $ is $ \beta $-smooth, i.e., $ \norm{ \g F^k_i(\w) - \g F^k_i(\w') } \le \beta\norm{ \w - \w' } $ for any $\w$ and $ \w' $.

%\begin{assumption}
%\label{Assumption}
%    For every $i$ and $k$, \circled{1} $ F^k_i $ is convex\footnote{We will empirically show that a non-convex function works well in Section \ref{sec:Evaluation}.};
%    \circled{2} $ F^k_i $ is $ \rho $-Lipschitz, i.e., $ \norm{ F^k_i(\w) - F^k_i(\w') } \le \rho\norm{ \w - \w' } $ for any $\w$ and $ \w' $; and
%    \circled{3} $ F^k_i $ is $ \beta $-smooth, i.e., $ \norm{ \g F^k_i(\w) - \g F^k_i(\w') } \le \beta\norm{ \w - \w' } $ for any $\w$ and $ \w' $. \hfill \qed
%\end{assumption}

Under this assumption, Lemma~\ref{lemma:Assumption} holds for the group and global loss functions.% $F_k$, which is the loss function for a node group, is additionally considered here unlike \citet{AdaptiveFL}.

\begin{lemma}
\label{lemma:Assumption}
    $F$ and $F^k$ are convex, $ \rho $-Lipschitz, and $ \beta $-smooth.
\vspace*{-0.2cm}
\end{lemma}
\vspace*{-0.3cm}
\begin{proof}
It is straightforward from the aforementioned assumption and the definitions of $F$ and $F^k$ in Eq.~\eqref{eq:loss_GroupFL}.
\vspace*{-0.2cm}
\end{proof}
We introduce two types of intervals depending on the learning level: a \emph{g\textbf{\underline{r}}oup interval}, $ \bkt{r} \triangleq \bkt{ \p{r-1}\uptau_1, r\uptau_1 } $, indicates an interval between two successive \emph{group} aggregations, and a \emph{g\textbf{\underline{l}}obal interval}, $ \bkt{l} \triangleq \bkt{ \p{l-1}\uptau_1\uptau_2, l\uptau_1\uptau_2 } $, indicates an interval between two successive \emph{global} aggregations.% Formally, the g\textbf{\underline{r}}oup interval $ \bkt{r} $ denotes  $ \bkt{ \p{r-1}\uptau_1, r\uptau_1 } $, which ranges from the $\p{r-1}$-th to the next group aggregation; the g\textbf{\underline{l}}obal interval $ \bkt{l} $ denotes $ \bkt{ \p{l-1}\uptau_1\uptau_2, l\uptau_1\uptau_2 } $, which ranges from the $\p{l-1}$-th to the next global aggregation.

Next, we introduce the notion of \emph{group-based virtual learning} in Definition~\ref{def:v_central}, where training data is assumed to exist on a \emph{virtual} central repository for each model.%This notion is used to bridge the \emph{local-to-group divergence}\,(i.e., the divergence between a local model and a group model) in a group interval and the \emph{group-to-global divergence}\,(i.e., the divergence between a group model and a global model) in a global interval.

\begin{definition}[Group-Based Virtual Learning]
\label{def:v_central}
    Given a certain group membership $\z$, for any $k$, $ \bkt{r} $, and $ \bkt{l} $, the \emph{virtual group model} $ \v^k_{\bkt{r}} $ and \emph{virtual global model} $ \v_{\bkt{l}} $ are updated by performing gradient descent steps on the centralized data examples for $\mathcal{N}^k$ and $\mathcal{N}$, respectively, and synchronized with the federated group model $ \w^k $ and the global model $\w$ at the beginning of each interval, as in Eq.~\eqref{eq:v_central}.
    \begin{equation}
    \label{eq:v_central}
    \resizebox{\columnwidth}{!}{$
    \begin{aligned}
        \v^k_{\bkt{r}}\p{t} \hspace{-1pt} \triangleq \hspace{-2pt} \begin{dcases}
            \w^k\p{t} & \normal{if} ~ t = \p{m \hspace{-1pt} - \hspace{-1pt} 1}\uptau_1, \\
            \v^k_{\bkt{r}}\p{t \hspace{-1pt} - \hspace{-1pt} 1} \hspace{-1pt} - \hspace{-1pt} \eta\g F^k\p{\v^k_{\bkt{r}}\p{t \hspace{-1pt} - \hspace{-1pt} 1}} \hspace{-6pt} & \normal{otherwise}
        \end{dcases} \\
        \v_{\bkt{l}}\p{t} \hspace{-1pt} \triangleq \hspace{-2pt} \begin{dcases}
            \w\p{t} & \normal{if} ~ t = \p{n \hspace{-1pt} - \hspace{-1pt} 1}\uptau_1\uptau_2, \\
            \v_{\bkt{l}}\p{t \hspace{-1pt} - \hspace{-1pt} 1} \hspace{-1pt} - \hspace{-1pt} \eta\g F\p{\v_{\bkt{l}}\p{t \hspace{-1pt} - \hspace{-1pt} 1}} \hspace{-6pt} & \normal{otherwise} ~~~~~~~~ \qed
        \end{dcases}
    \end{aligned}
    $} 
    \end{equation} 
\end{definition}

To facilitate the interpretation, \autoref{fig:wv} shows how a virtual model $\v$ is updated, following Definition~\ref{def:v_central}. For example, $\v_{\left[ l \right]} $ starts diverging from $\w$ after $(l-1)\uptau_1\uptau_2$ and becomes synchronized with $\w$ at $l\uptau_1\uptau_2$.
% To facilitate the interpretation, \autoref{fig:wv} shows how a virtual model $\v$ diverges and is synchronized with a federated model $\w$, following Definition~\ref{def:v_central}. For example, $\v_{\left[ l \right]} $ starts diverging from $\w$ after $(l-1)\uptau_1\uptau_2$ and becomes synchronized with $\w$ at $l\uptau_1\uptau_2$.

\begin{figure}[t!]
    \centering
    \includegraphics[width=\columnwidth]{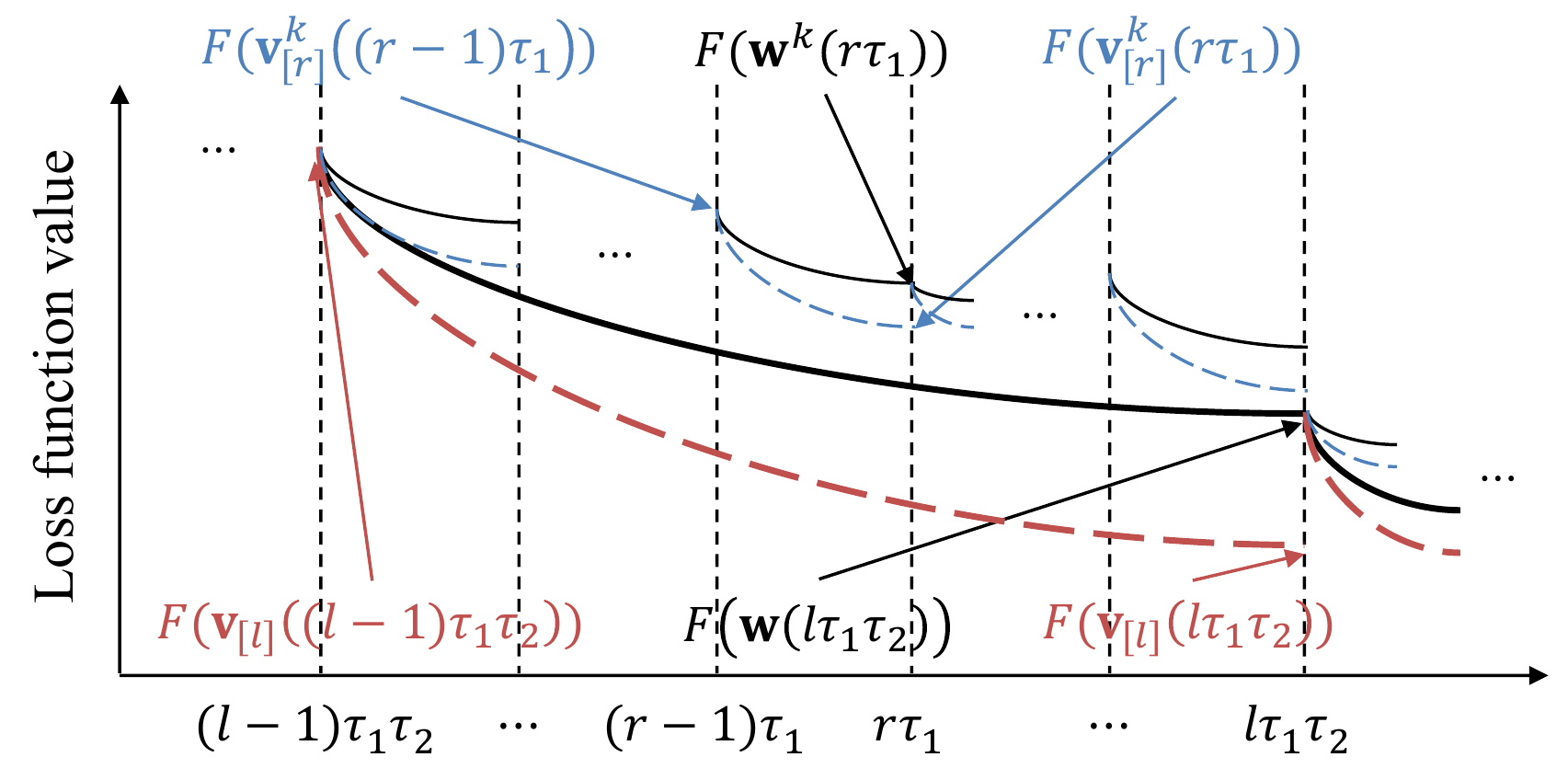}
    \vspace*{-0.6cm}
    \caption{Illustration of loss divergence and synchronization between $ \w^k $ and $ \v^k_{\left[ r \right]} $ and between $\w$ and $ \v_{\left[ l \right]} $.}
    \label{fig:wv}
    \vspace*{-0.4cm}
\end{figure}

Then, we formalize \emph{group-based gradient divergence} in Definition~\ref{def:GradientDivergence}
% by extending the gradient divergence\,\citep{AdaptiveFL}
that models the impact of the difference in data distributions across nodes on federated learning.
\begin{definition}[Group-Based Gradient Divergence]
\label{def:GradientDivergence}
    Given a certain group membership $\z$, for any $i$, $k$, and $\w$, $ \delta^k_i $ is defined as the gradient difference between the $i$-th local loss and the $k$-th group loss; $ \Delta^k $ is defined as the gradient difference between the $k$-th group loss and the global loss, which can be expressed as Eq.~\eqref{eq:GradientDivergence}.
    \begin{equation}
    \begin{split}
    \label{eq:GradientDivergence}
        \delta^k_i \triangleq \norm{ \g F^k_i\p{\w} - \g F^k\p{\w} }, \\
        \Delta^k \triangleq \norm{ \g F^k\p{\w} - \g F\p{\w} }
    \end{split}
    \end{equation}
    Then, the \textit{local-to-group divergence} $\delta$ and the \textit{group-to-global divergence} $\Delta$ are formulated as Eq.\ \eqref{eq:GradientDivergenceSum}.
    \begin{equation}
    \label{eq:GradientDivergenceSum}    
        \delta \triangleq \sum_{ k \in \gK }{ \sum_{ i \in \gN^k }{ \frac{ \abs{\gD^k_i} }{ \abs{\gD} } \delta^k_i } },~~  
        \Delta \triangleq \sum_{ k \in \gK }{ \frac{ \abs{\gD^k} }{ \abs{\gD} } \Delta^k } \qed
    \end{equation}
\end{definition}

\vspace*{-0.4cm}

\subsection{Convergence of the Group Federated Learning}
\label{sec:Convergence}
We provide a proof sketch for the convergence of the global loss $F(\w(T))$ in Appendix~\ref{sec:ProofSketch}, which derives Theorem \ref{thm:ConvUpperBound}.
\begin{theorem}
\label{thm:ConvUpperBound}
Let $ \omega \triangleq \min_{q}{ \frac{ 1 }{ \norm{ \v_{\bkt{l}}\p{\p{l-1}\uptau_1\uptau_2} - \w^* }^2 } } $. When $ \eta \le \frac{ 1 }{ \beta } $, the convergence upper bound of group federated learning after $ T $ steps can be expressed as Eq.~\eqref{eq:ConvUpperBound}.
\begin{equation}
\label{eq:ConvUpperBound}
\begin{split}
    & F\p{\w\p{T}} - F\p{\w^*} \le \frac{1}{2\uptau_1 \uptau_2 \eta \omega} \\
    & + \rho \p{ \frac{ \delta }{ \beta } \p{ \p{ \eta\beta + 1 }^{\uptau_1 } - 1 } + \frac{ \Delta }{ \beta } \p{ \p{ \eta\beta + 1 }^{ \uptau_1\uptau_2 } - 1 } }
\end{split}
\end{equation}
\end{theorem}
\vspace*{-1.0cm}
\begin{proof}
Please refer to Appendix~\ref{sec:Proof_ConvUpperBound} for details. ~~~~~~$\Box$ \let\qed\relax
\end{proof}
\vspace*{-0.2cm}

From Eq.~\eqref{eq:ConvUpperBound}, it is straightforward to see that the optimality gap is dominantly affected by $\uptau_1$, $\uptau_2$, $\delta$, and $\Delta$.  Therefore, given small values, the convergence is guaranteed.

\vspace*{-0.2cm}
\subsection{Theoretical Analysis for the Main Problem}
\label{sec:Interpretation}
Based on the convergence analysis, we interpret the IID and communication-aware group federated learning as follows.%, in terms of the optimization parameters, $\uptau_1$, $\uptau_2$, $\abs{\gK}$, and $\z$.
\begin{remark}[Dominance of $\Delta$]
\label{remark:Delta}
The IID objective\,($\min{F(\w(T))}$) in Eq.~\eqref{eq:Problem} is the same as minimizing $ F(\w(T)) - F(\w^*) $ in Eq.~\eqref{eq:ConvUpperBound} because $F(\w^*)$ is a constant. Thus, given $\uptau_1$, $\uptau_2$, and $\abs{\gK}$, because $\Delta$ is the most dominant factor in Eq.~\eqref{eq:ConvUpperBound}, it is important for the IID objective to reduce $\Delta$ by changing $\z$.
\end{remark}
\begin{remark}[Dominance of $d_{group}$]
\label{remark:Comm}
For the communication objective\,($ \min{ d_{group}\p{\uptau_2-1} + d_{global} } $), because $d_{global}$ is not affected by a certain node grouping from the definition of global aggregation in Eq.~\eqref{eq:w_GroupFL}, given $\uptau_1$, $\uptau_2$, and $\abs{\gK}$, it is important for the communication objective to reduce $d_{group}$ by changing $\z$.
\end{remark}

In conclusion, we establish the \textbf{IID and communication-aware grouping} principle: a group federated learning algorithm should group nodes by simultaneously minimizing $\Delta$ and $d_{group}$ to maximize both accuracy and efficiency.
%
% \begin{remark}[Tradeoff between Accuracy and Efficiency]
% \label{remark:Tradeoff}
% As $ \abs{\gK} \rightarrow 1 $, $ \Delta \rightarrow 0 $ because the data distribution of a group becomes closer to the global IID distribution, whereas $ d_{group} \rightarrow d_{global} $\,($ d_{group} \leq d_{global} $) because the total amount of communication within a group becomes larger with the increased number of nodes per group.
% \end{remark}

\section{Optimization Algorithm: FedAvg-IC}
\label{sec:Optimization}
To solve Eq.~\eqref{eq:Problem}, an efficient heuristic algorithm is essential because the grouping problem itself is NP-Hard with the complexity of $ O(\abs{\gN}^{\abs{\gK}}) $. In this regard, we propose a novel algorithm called \textbf{FedAvg-IC}\,(\emph{\underline{Fed}erated \underline{Av}era\underline{g}ing with \underline{I}ID and \underline{C}ommunication-Aware Grouping}).% in Section \ref{sec:AlgorithmDesc} together with a \emph{combined aggregation} technique to further reduce the amount of communication in Section \ref{sec:Advanced}.

\subsection{Algorithm Description}
\label{sec:AlgorithmDesc}

FedAvg-IC aims at quickly finding an accurate global model based on the near-optimal node grouping that follows the IID and communication-aware grouping principle, for which we adopt the \emph{k-medoids} algorithm\,\citep{park2009simple}. % because the problem is a discrete optimization.
The node grouping involves assigning each node to the closest medoid node and updating a representative medoid node for each group. Here, the distance is measured by the cost functions defined as follows.

{\bf \underline{Assign Cost}}: To evaluate the cost of assigning the $i$-th node to the $k$-th group, we model the IID cost\,($\textsc{Cost}_{A,iid}$) and communication cost\,($\textsc{Cost}_{A,comm}$) using $\Delta^k$ in Eq.~\eqref{eq:GradientDivergence} and the hop distance between the $i$-th node and $i_k$-th medoid node, respectively, as shown in Eq.~\eqref{eq:Cost_A}.
\vspace*{-0.1cm}
\begin{equation}
\label{eq:Cost_A}
\resizebox{.7\columnwidth}{!}{$
\begin{aligned}
    & \textsc{Cost}_{A,iid}\p{i, k} \triangleq \Delta^k ~~ \normal{where} ~~ i \in \gN^k \\
    & \textsc{Cost}_{A,comm}\p{i, k} \triangleq \textsc{HopDistance}\p{i, i_k}
\end{aligned}
$}
\vspace*{-0.1cm}
\end{equation}
{\bf \underline{Update Cost}}: To evaluate the cost of selecting the $i$-th node in the $k$-th group as a new medoid for the group, we model the IID cost\,($\textsc{Cost}_{U,iid}$) and the communication cost\,($\textsc{Cost}_{U,comm}$) by the local-to-global divergence of the $i$-th node and the sum of hop distances to all other nodes in the group, respectively, as shown in Eq.~\eqref{eq:Cost_U}.
\vspace*{-0.2cm}
\begin{equation}
\label{eq:Cost_U}
\resizebox{.75\columnwidth}{!}{$
\begin{aligned}
    & \textsc{Cost}_{U,iid}\p{i, k} \triangleq ~ \frac{ \abs{\gD^k_i} }{ \abs{\gD} } \norm{ \g F^k_i\p{\w} - \g F\p{\w} } \\
    & \textsc{Cost}_{U,comm}\p{i, k} \triangleq ~ \sum_{ j \in \gN^k }{ \textsc{HopDistance}\p{i, j} }
\end{aligned}
$}
\vspace*{-0.1cm}
\end{equation}
{\bf \underline{Combined Cost}}: %Based on the weighted sum method, 
Given $X \in \set{A, U}$, $\textsc{Cost}_{X,iid}$\,(IID cost) and $\textsc{Cost}_{X,comm}$\,(communication cost) are combined into a single cost, as shown in Eq.~\eqref{eq:Cost}.
\vspace*{-0.1cm}
\begin{equation}
\resizebox{.85\columnwidth}{!}{$
\label{eq:Cost}
    \textsc{Cost}_{X} \triangleq \alpha_{iid}\frac{\textsc{Cost}_{X,iid}}{C_{X,iid}} + \alpha_{comm}\frac{\textsc{Cost}_{X,comm}}{C_{X,comm}}
$}
\vspace*{-0.1cm}
\end{equation}
$\alpha$ is the weight, and $C$ is the normalizing constant\footnote{$C$ is set to be the first cost value in the optimization process\,\citep{grodzevich2006normalization}.}.

\begin{algorithm}[t!]
\caption{FedAvg-IC}
\label{alg:FedAvg-IC}
\begin{algorithmic}[1]
\Require{$ \gN, T, \uptau^{0}_1, \uptau^{0}_2, \abs{\gK} $}
\Ensure{$ \w\p{T} $}
\State{Initialize $ \w\p{0} $ and $\z$ randomly, $ \uptau_1 \assign 1, ~ \uptau_2 \assign 1 $}
\State{$ \bkt{ \w^k_i\p{0} }_{ i \in \gN } \assign \w\p{0} $} \Comment{Initial global broadcast\footnotemark}
\For{$ t \assign 1,2,\cdots,T $}
    \ParForeach{$ i \in \gN $} \Comment{Local}
        \State{$ \w^k_i\p{t} \assign \w^k_i\p{t-1} - \eta \g F^k_i\p{\w^k_i\p{t-1}} $}
    \EndParForeach
    \If{$ \p{ t - 1 } \normal{ mod } \uptau_1\uptau_2 \ne 0 $} \Comment{Group}%\footnotemark}
        \ParForeach{$ k \in \gK $}
            \State{$ \bkt{ \w^k }_{ i_k } \assign \sum_{ i \in \gN^k }{ \frac{ \abs{\gD^k_i} }{ \abs{\gD^k} } \bkt{ \w^k_i\p{t} }_{ i } } $}
            \State{$ \bkt{ \w^k_i\p{t} }_{ i \in \gN^k } \assign \bkt{ \w^k }_{ i_k } $}
        \EndParForeach
    \EndIf
    \If{$ \p{ t - 1 } \normal{ mod } \uptau_1\uptau_2 = 0 $} \Comment{Global}
        \State{$ \w \assign \sum_{ i \in \gN }{ \frac{ \abs{\gD^k_i} }{ \abs{\gD} } \bkt{ \w^k_i\p{t} }_{ i } } $}
        \State{$ \bkt{ \w^k_i\p{t} }_{ i \in \gN } \assign \w $}
        \If{$ \gN $ is not grouped}
            \State{$ \z \assign \textsc{Node\_Grouping}\p{ \z } $}
            \State{$ \p{ \uptau_1, \uptau_2 } \assign \p{ \uptau^{0}_1, \uptau^{0}_2 } $}
        \EndIf
    \EndIf
\EndFor
\Function{\textsc{Node\_Grouping}}{$\z$}
    \State{Select random medoid nodes $\gN_{m}$}
    \State{$ \z \assign \p{ \argmin_{ k \in \gK }{ \textsc{Cost}_{A}\p{i, k} } | \forall i \in \gN } $}
    \Until{the last $\textsc{Cost}_{A}$ is steady}
        \State{$ \gN_{m} \assign \p{ \argmin_{ i \in \gN^k }{ \textsc{Cost}_{U}\p{i, k} } | \forall k \in \gK } $}
        \State{$ \z \assign \p{ \argmin_{ k \in \gK }{ \textsc{Cost}_{A}\p{i, k} } | \forall i \in \gN } $}
    \EndUntil
    \Return{$\z$}
\EndFunction
\end{algorithmic}
\end{algorithm}
\addtocounter{footnote}{-1}
\stepcounter{footnote}\footnotetext{$ \bkt{X}_{i} $ denotes the variable reference of $ X $ at a node $ i $.}
% \stepcounter{footnote}\footnotetext{$ i_k $ is the node index of the group aggregation server at a group $ k $ that is determined to be the one with the smallest sum of hop distances to all other nodes in the group.}

Algorithm \ref{alg:FedAvg-IC} shows the overall procedure of FedAvg-IC. It takes the set of node indices $ \gN $, the final time $ T $, the learning steps $\uptau^0_1$ and $\uptau^0_2$, and the number of node groups $ \abs{\gK} $ as the input and returns the final global model $ \w\p{T} $ as the output. It begins by initializing the global model and group membership randomly\,(Line 1). Then, the global model is broadcast to all nodes\,(Line 2). 
% Before the first global aggregation, group federated learning relies on this initial group membership.
Then, the local update is performed at each node\,(Lines 4--5); each group model is learned by aggregating all local models in the group and then broadcast back to all nodes\,(Lines 6--9); the global model is learned by aggregating all local models and then broadcast back to all nodes\,(Lines 10--12). After the first global aggregation, the group membership $\z$ is updated\,(Line 14). 
%This updated $\z$ is maintained through subsequent global intervals.
Overall, Lines 3--15 repeat for $ T $ steps.

The \textsc{Node\_Grouping} function attempts to find a group membership $\z$ that reduces the combined cost in Eq.~\eqref{eq:Cost} to the extent possible. For this purpose, it begins by selecting random medoid nodes $\gN_{m}$ of size $\abs{\gK}$. Then, it iteratively updates $\z$ by minimizing $\textsc{Cost}_{A}$ in Eq.~\eqref{eq:Cost_A} for all nodes and $\textsc{Cost}_{U}$ in Eq.~\eqref{eq:Cost_U} for all groups until the cost is steady\,(Lines 19--21).
\section{Evaluation}
\label{sec:Evaluation}

\subsection{Experimental Setting}
\label{sec:ExperimentalSetting}

{\bf \underline{Configuration}}:
We developed a federated learning simulator to extensively evaluate the performance of various algorithms, models, datasets, and networks based on TensorFlow 1.14.0. Please refer to Appendix~\ref{sec:ExperimentalSettingDetails} for details.

{\bf \underline{Algorithms}}:
We compared the following \emph{three} algorithms.
\begin{itemize}[noitemsep,leftmargin=10pt,nosep]
\item
    \textbf{FedAvg}\,\citep{FedAvg}, which is used as a baseline, does not consider node grouping at all.
\item
    \textbf{HierFAVG}\,\citep{HierFAVG} groups nodes by network edges to facilitate communication between nodes.
\item
    \underline{\textbf{FedAvg-IC}} groups nodes by minimizing both IID and communication costs. We also considered FedAvg-IC that only minimizes either IID or communication cost as FedAvg-I or FedAvg-C, respectively.
\end{itemize}

{\bf \underline{Datasets}}:
We used \emph{four} datasets, \circled{1} MNIST-O\,\citep{MNIST-O}, \circled{2} MNIST-F\,\citep{MNIST-F}, \circled{3} FEMNIST\,\citep{caldas2018leaf}, and \circled{4} CelebA\,\citep{celeba}, which consist of 70,000, 70,000, 78,353, and 10,014 examples, respectively. The ratio of train/validation/test examples was 3:1:1, as suggested by \citet{caldas2018leaf}.

\begin{figure*}[t!]
\centering
    \includegraphics[height=1.0cm]{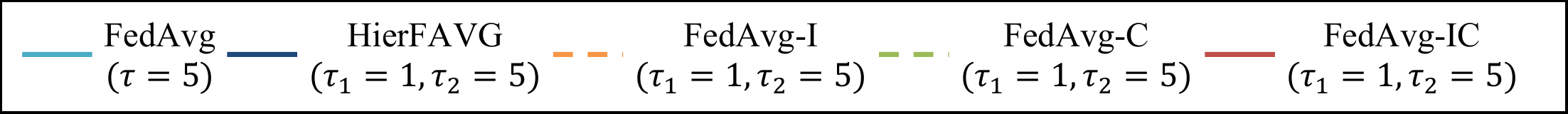} \\
    \begin{minipage}{.795\linewidth}
        \centering
        \includegraphics[height=3.1cm]{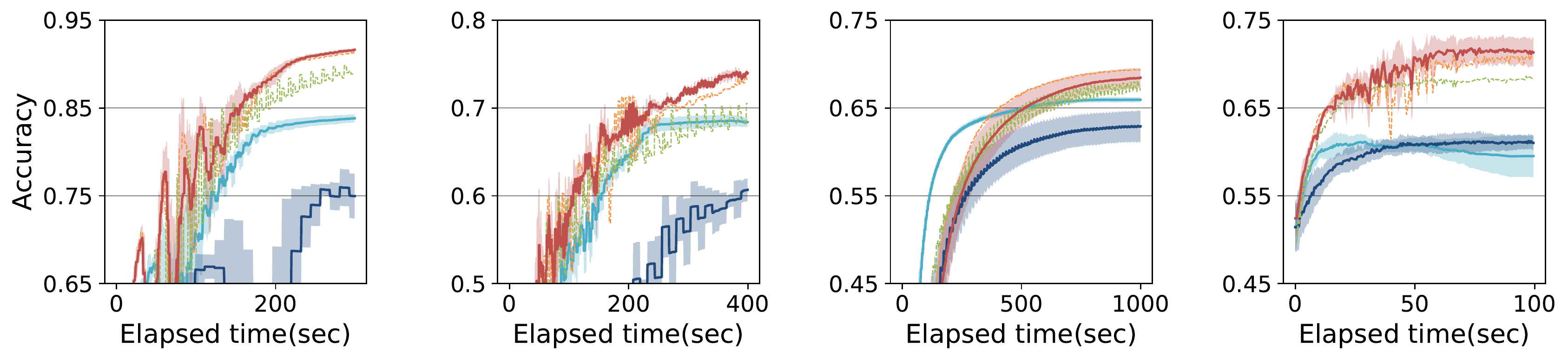}
        \vspace{-10pt} \\ 
        \begin{minipage}[t]{.28\linewidth}
        \centering
        \subcaption{MNIST-O.}\label{fig:accuracy-cnn-mnist_mnist-o}
        \end{minipage}
        \begin{minipage}[t]{.22\linewidth}
        \centering
        \subcaption{MNIST-F.}\label{fig:accuracy-cnn-mnist_mnist-f}
        \end{minipage}
        \begin{minipage}[t]{.26\linewidth}
        \raggedright
        \subcaption{FEMNIST.}\label{fig:accuracy-cnn-mnist_femnist}
        \end{minipage}
        \begin{minipage}[t]{.22\linewidth}
        \centering
        \subcaption{CelebA.}\label{fig:accuracy-cnn-celeba_celeba}
        \end{minipage}
        \vspace{-10pt}
        \caption{Test accuracy of the CNN on four datasets with \emph{Dtt} according to elapsed time.\newline}
        \label{fig:AccuracyTime}
    \end{minipage}
    \hspace{-10pt}
    \begin{minipage}{.2\linewidth}
        \begin{minipage}[t]{\linewidth}
        \centering
        \includegraphics[height=3.1cm]{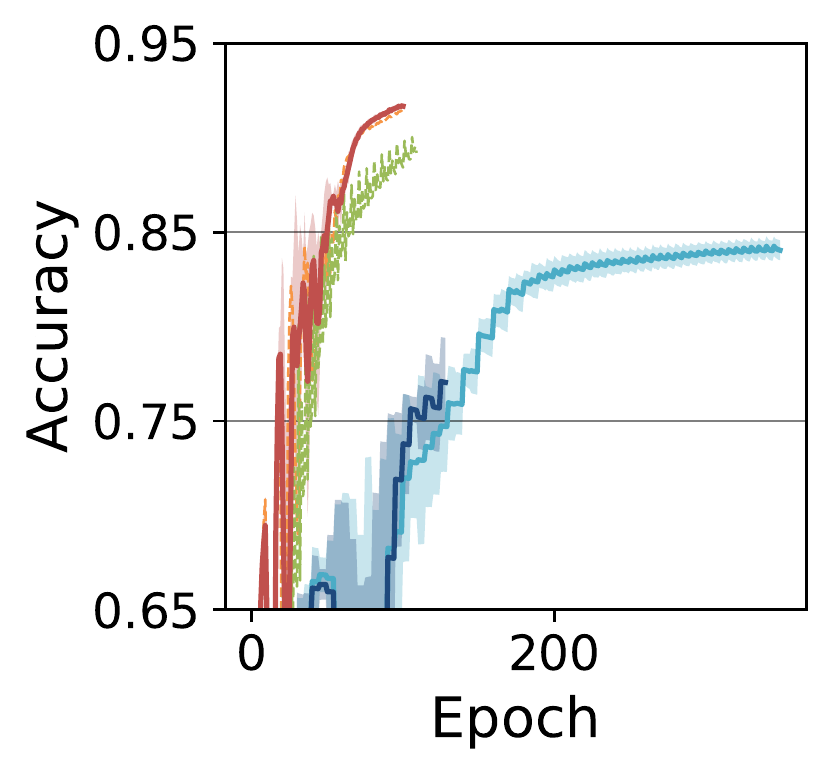}
        \vspace{-6pt} \\ 
        \subcaption{MNIST-O.}\label{fig:AccuracyEpoch}
        \end{minipage}
        \caption{Test accuracy according to epochs.}
    \end{minipage}
\vspace*{-0.5cm}
\end{figure*}

\begin{table}[h!]
\caption{Class diversity across nodes and edges. An entry is the number of classes per node or edge.}
\centering
\small
\begin{tabular}{rrrrrrr} \toprule
    & ~\textbf{\textit{Dtt}} & ~\textbf{\textit{Dtq}} & ~\textbf{\textit{Dth}} & ~\textbf{\textit{Dqq}} & ~\textbf{\textit{Dqh}} & ~\textbf{\textit{Dhh}} \\ \midrule
    Node & \emph{tenth} & \emph{tenth} & \emph{tenth} & \emph{quarter} & \emph{quarter} & \emph{half} \\
    Edge & \emph{tenth} & \emph{quarter} & \emph{half} & \emph{quarter} & \emph{half}  & \emph{half} \\ \bottomrule
\end{tabular}
\label{table:ClassDiversity}
\vspace*{-0.4cm}
\end{table}

{\bf \underline{Data Distribution}}:
To simulate a wide range of non-IIDness, we designed \emph{six} cases of class diversity on each node and edge, as shown in \autoref{table:ClassDiversity}. For example, in the \emph{Dtq} setting, only a tenth of the classes can exist per node, and a quarter of the classes can exist per edge.
% The \emph{Dtt} setting has the highest degree of non-IIDness, while the \emph{Dhh} setting has the lowest degree of non-IIDness. In addition, to incorporate the unbalanced property\,\citep{FedAvg}, the number of classes per node or edge and the number of data examples per node were randomly sampled from the normal distribution.

{\bf \underline{Models}}:
We used \emph{three} training models, \circled{1} the softmax regression\,(SR), \circled{2} the 2 layered perceptron neural network\,(2NN), and \circled{3} the convolutional neural network\,(CNN). Please refer to Appendix~\ref{sec:ExperimentalSettingDetails} for details.

{\bf \underline{Methodology}}:
Deterministic gradient descent\,(DGD) was used for the SR to solve convex problems, and stochastic gradient descent\,(SGD) was used for the 2NN and the CNN to solve non-convex problems. We evaluate each algorithm five times and report the average with standard deviation.

\subsection{Accuracy Results}
\label{sec:AccuracyResults}

\autoref{fig:AccuracyTime} and \autoref{fig:AccuracyEpoch} show the test accuracy of three federated learning algorithms on a non-IID\,(\emph{Dtt}) dataset according to the elapsed time and epoch, respectively. Overall, FedAvg-IC outperformed FedAvg by up to 17.4\% (\autoref{fig:accuracy-cnn-celeba_celeba}) and HierFAVG by up to 22.2\% (\autoref{fig:accuracy-cnn-mnist_mnist-f}). In \autoref{fig:AccuracyEpoch}, the algorithms that considered communication such as FedAvg-IC and FedAvg-C outperformed the others. The higher accuracy of FedAvg-IC is attributed to a decreased $\Delta$ in the IID cost in Eq.~\eqref{eq:Cost}. Please refer to Appendix~\ref{sec:EvaluationDetails} for details.

\subsection{Efficiency Results}
\label{sec:EfficiencyResults}

\begin{table}[h!]
\caption{Elapsed time\,(in seconds) of the algorithms on the non-IID and IID datasets at the final test accuracy of the baseline FedAvg within a given time, where the final accuracy is specified in parenthesis next to each model name.}
\centering
\small
\begin{tabular} {>{\centering}p{0.22\columnwidth}>{\centering\arraybackslash}p{0.21\columnwidth}>{\centering\arraybackslash}p{0.21\columnwidth}>{\centering\arraybackslash}p{0.21\columnwidth}} \toprule
    & \multicolumn{3}{c}{Non-IID\,(Dtt)} \\
    & SR(84\%) & 2NN(73\%) & CNN(83\%) \\ \midrule
    \makecell{FedAvg} & 50 & 300 & 300 \\
    HierFAVG & 29(1.7x) & $-$ & $-$ \\
    \makecell{\textbf{FedAvg-IC}} & 6(8.3x) & 47(6.4x) & 149(2.0x) \\ \toprule
    & \multicolumn{3}{c}{IID\,(Dhh)} \\
    & SR(86\%) & 2NN(90\%) & CNN(96\%) \\ \midrule
    \makecell{FedAvg} & 100 & 600 & 600 \\
    HierFAVG & 29(3.4x) & 468(1.3x) & $-$ \\
    \makecell{\textbf{FedAvg-IC}} & 18(5.6x) & 291(2.1x) & 543(1.1x) \\ \bottomrule
\end{tabular}
\label{table:Efficiency}
\vspace*{-0.4cm}
\end{table}

\autoref{table:Efficiency} shows the elapsed time and speedup on the most non-IID\,(\emph{Dtt}) and IID\,(\emph{Dhh}) datasets. In terms of the elapsed time, FedAvg-IC outperformed FedAvg and HierFAVG by up to 8.3 times and 4.8 times, respectively. Even though HierFAVG is in favor of communication efficiency, because the edge-based learning of HierFAVG degrades the accuracy in non-IID settings, it did not reach the target accuracy for the 2NN and the CNN.
The faster convergence speed of FedAvg-IC is attributed to a decreased $d_{group}$ in Eq.~\eqref{eq:Cost} as well as a decreased communication data size by the combined aggregation. Please refer to Appendix~\ref{sec:EvaluationDetails}.
\section{Conclusion}
In this paper, we proposed a novel framework of \emph{IID and communication-aware group federated learning} to address both the non-IID and limited communication challenges simultaneously. Our formal convergence analysis led to the \textrm{IID and communication-aware grouping} principle that is incorporated into our optimization algorithm \emph{FedAvg-IC}. Extensive experiments were performed using our own federated learning simulator, and the results demonstrated that FedAvg-IC outperformed HierFAVG by up to $22.2\%$ in terms of test accuracy and FedAvg by up to $8.3$ times in terms of convergence speed. Overall, we believe that our framework has made important steps towards accurate and fast federated learning.

\bibliography{08-reference}
\bibliographystyle{icml2020}

\appendix
\section{Convergence of the Group Federated Learning}
\subsection{Proof Sketch}
\label{sec:ProofSketch}
We sketch the proof for the convergence of the global loss $F(\w(T))$ in Eq.~\eqref{eq:loss_GroupFL} through the following three steps.
\begin{itemize}[noitemsep,leftmargin=10pt,nosep]
\item
    \textbf{Step 1}\,(Local Learning Divergence): For a \emph{group} interval $ \bkt{r} $, we find the loss divergence between a local model and a virtual group model, $ F(\w^k_i(t)) - F(\v^k_{\bkt{r}}(t)) $.
\item
    \textbf{Step 2}\,(Group Learning Divergence): For a \emph{global} interval $ \bkt{l} $, we find the loss divergence between a virtual group model and a virtual global model, $ F(\v^k_{\bkt{r}}(t)) - F(\v_{\bkt{l}}(t)) $. Then, by combining the aforementioned two loss divergences for all local models $ \w^k_i $, we obtain the loss divergence between a federated global model and a virtual global model, $ F(\w(t)) - F(\v_{\bkt{l}}(t)) $.
\item
    \textbf{Step 3}\,(Global Learning Divergence): For \emph{all} global intervals, by combining $ F(\w(t)) - F(\v_{\bkt{l}}(t)) $ from Step 2 with the loss divergence between a virtual global model and the optimal model, $ F(\v_{\bkt{l}}(t)) - F(\w^*) $, we finally obtain $ F(\w(T)) - F(\w^*) $.
\end{itemize}

Corresponding to \underline{Steps 1 and 2} of the proof sketch, Lemma~\ref{lemma:w_vq} gives an upper bound between a federated global model $ \w\p{t} $ and a virtual global model $ \v_{\bkt{l}}\p{t} $.
\begin{lemma}
\label{lemma:w_vq}
For any global interval $ \bkt{l} $ and $ t \in \bkt{l} $, if $ F^k_i $ is $ \beta $-smooth for every $i$ and $k$ in Eq.\ \eqref{eq:GradientDivergence}, then Eq.\ \eqref{eq:w_vq_bound} holds.
\begin{equation}
\label{eq:w_vq_bound}
\begin{split}
    &\norm{ \w\p{t} - \v_{\bkt{l}}\p{t} } \\
    &\le \frac{ \delta }{ \beta } \p{ \p{ \eta\beta + 1 }^{ \uptau_1 } - 1 } + \frac{ \Delta }{ \beta } \p{ \p{ \eta\beta + 1 }^{ \uptau_1\uptau_2 } - 1 }
\end{split}
\end{equation}
\end{lemma}
\begin{proof}
Please refer to Appendix~\ref{sec:Proof_Lemma} for details.
\end{proof}

Finally, corresponding to \underline{Step 3} of the proof sketch, Theorem \ref{thm:ConvUpperBound} is derived from Lemma~\ref{lemma:w_vq}.

\subsection{Proof of Lemma~\ref{lemma:w_vq}}
\label{sec:Proof_Lemma}
To prove Lemma~\ref{lemma:w_vq}, we introduce an auxiliary lemma (Lemma~\ref{lemma:wi-vq}).

\begin{lemma}
\label{lemma:wi-vq}
For any $ \bkt{r} $, $ \bkt{l} $, and $ t \in \bkt{ \p{r-1}\uptau_1, r\uptau_1 } \subset \bkt{ \p{l-1}\uptau_1\uptau_2, l\uptau_1\uptau_2 } $, an upper bound of the norm of the difference between a local model and the virtual global model can be expressed as Eq.~\eqref{eq:node-virtual-global}.
\begin{equation}
\label{eq:node-virtual-global}
\resizebox{\columnwidth}{!}{$
\begin{multlined}
    \norm{ \w^k_i\p{t} - \v_{\bkt{l}}\p{t} } \\
    \le \frac{ \delta^k_i }{ \beta } \p{ \p{ \eta\beta + 1 }^{ t - \p{r-1}\uptau_1 } - 1 } + \frac{ \Delta^k }{ \beta } \p{ \p{ \eta\beta + 1 }^{ t - \p{l-1}\uptau_1\uptau_2 } - 1 }
\end{multlined}
$}
\end{equation}
\end{lemma}
\begin{proof}
From the triangle inequality, one can simply derive Eq.~\eqref{eq:wvdistc}.
\begin{equation}
\resizebox{0.8\columnwidth}{!}{$
\begin{multlined}
    \norm{ \w^k_i\p{t} - \v_{\bkt{l}}\p{t} } \\
    = \norm { \w^k_i\p{t} - \v^k_{\bkt{r}}\p{t} + \v^k_{\bkt{r}}\p{t} - \v_{\bkt{l}}\p{t} } \hspace{0.1cm}\\
    \le \norm{ \w^k_i\p{t} - \v^k_{\bkt{r}}\p{t} } + \norm{ \v^k_{\bkt{r}}\p{t} - \v_{\bkt{l}}\p{t} } \label{eq:wvdistc}
\end{multlined}
$}
\end{equation}
To conclude this proof, it thus suffices to show Eq.~\eqref{eq:wvdist1} and \eqref{eq:wvdist2}.
\begin{align}
    \norm{ \w^k_i\p{t} - \v^k_{\bkt{r}}\p{t} } \le & \frac{ \delta^k_i }{ \beta } \p{ \p{ \eta\beta + 1 }^{ t - \p{r-1}\uptau_1 } - 1 }\quad \label{eq:wvdist1} \\
    \norm{ \v^k_{\bkt{r}}\p{t} - \v_{\bkt{l}}\p{t} } \le & \frac{ \Delta^k }{ \beta } \p{ \p{ \eta\beta + 1 }^{ t - \p{l-1}\uptau_1\uptau_2 } - 1 } \label{eq:wvdist2}
\end{align}
Then, by putting Eq.~\eqref{eq:wvdist1} and \eqref{eq:wvdist2} into Eq.~\eqref{eq:wvdistc}, we can confirm Lemma~\ref{lemma:wi-vq}.

Both Eq.~\eqref{eq:wvdist1} and \eqref{eq:wvdist2} can be easily drawn from the $\beta$-smooth property of $F_i^k$ and $F^k$. From Eq.~(4) and (6), we can derive Eq.~\eqref{eq:node-virtual-group}. % \eqref{eq:w_GroupFL} and \eqref{eq:v_central}
\begin{equation}
\label{eq:node-virtual-group}
\resizebox{\columnwidth}{!}{$
\begin{multlined}
    \norm{ \w^k_i\p{t} - \v^k_{\bkt{r}}\p{t} } \\
    \begin{aligned}[t]=\norm*{ \w^k_i\p{t-1} - \eta\g F^k_i\p{\w^k_i\p{t-1}} \hspace{2.2cm}\\
    - \v^k_{\bkt{r}}\p{t-1} + \eta\g F^k\p{\v^k_{\bkt{r}}\p{t-1}}} \end{aligned} \\
    % = \norm{ \w^k_i\p{t-1} - \eta\g F^k_i\p{\w^k_i\p{t-1}} - \v^k_{\bkt{r}}\p{t-1} + \eta\g F^k\p{\v^k_{\bkt{r}}\p{t-1}}} \\
    \le \norm{ \w^k_i\p{t-1} - \v^k_{\bkt{r}}\p{t-1} } \hspace{3.2cm} \\
    + \eta\norm{ \g F^k_i\p{\w^k_i\p{t-1}} - \g F^k_i\p{\v^k_{\bkt{r}}\p{t-1}} } \hspace{0.2cm} \\
    + \eta\norm{ \g F^k_i\p{\v^k_{\bkt{r}}\p{t-1}} - \g F^k\p{\v^k_{\bkt{r}}\p{t-1}} } \\
    \le \p{ \eta\beta + 1 }\norm{ \w^k_i\p{t-1} - \v^k_{\bkt{r}}\p{t-1} } + \eta\delta^k_i \hspace{1.3cm}
\end{multlined}
$}
\end{equation}

The last inequality stems from the $ \beta $-smoothness of $ F^k_i $ and Definition 2. %\ref{def:GradientDivergence}

Then, since $\w^k_i\p{t} = \w^k\p{t} = \v^k_{\bkt{r}}\p{t} $ at every group aggregation from Eq.~(4) and (6), Eq.~\eqref{eq:node-virtual-group} can be rewritten as Eq.~\eqref{eq:wi-delta}. %\eqref{eq:w_GroupFL} and \eqref{eq:v_central}
\begin{align}
    \norm{ \w^k_i\p{t} - \v^k_{\bkt{r}}\p{t} } \le & \eta\delta^k_i\sum^{ t - \p{r-1}\uptau_1 }_{ y = 1 }{ \p{ \eta\beta + 1 }^{ y - 1 } } \nonumber \\
    =& \frac{ \delta^k_i }{ \beta } \p{ \p{ \eta\beta + 1 }^{ t - \p{r-1}\uptau_1 } - 1 } \label{eq:wi-delta}
\end{align}
Analogously, one can derive Eq.~\eqref{eq:wvdist2}. This is the end of the proof of Lemma~\ref{lemma:wi-vq}.
\end{proof}
For all $t$ and $q$, from Lemma~\ref{lemma:wi-vq} and Jensen's inequality, Lemma~2 can be proven as in Eq.~\eqref{eq:w-v}.
\begin{equation}
\label{eq:w-v}
\resizebox{0.9\columnwidth}{!}{$
\begin{multlined}
    \norm{ \w\p{t} - \v_{\bkt{l}}\p{t} } \\
    \le \sum_{ k \in \gK } \sum_{ i \in \gN^k } \frac{ \abs{\gD^k_i} }{ \abs{\gD} } \norm{ \w^k_i\p{t} - \v_{\bkt{l}}\p{t} } \hspace{1.3cm} \\
    \le \frac{ \delta }{ \beta } \p{ \p{ \eta\beta + 1 }^{ \uptau_1 } - 1 } + \frac{ \Delta }{ \beta } \p{ \p{ \eta\beta + 1 }^{ \uptau_1\uptau_2 } - 1 }
\end{multlined}
$}
\end{equation}
Furthermore, since $F$ is $\rho$-Lipschitz, Eq.\ \eqref{eq:fw_fvq_bound} holds.
\begin{equation}
\label{eq:fw_fvq_bound}
\begin{split}
    &F\p{\w\p{t}} - F\p{\v_{\bkt{l}}\p{t}} \\
    &\le \rho\p{ \frac{ \delta }{ \beta } \p{ \p{ \eta\beta + 1 }^{ \uptau_1 } - 1 } + \frac{ \Delta }{ \beta } \p{ \p{ \eta\beta + 1 }^{ \uptau_1\uptau_2 } - 1 } }
\end{split}
\end{equation}

\subsection{Proof of Theorem~\ref{thm:ConvUpperBound}}
\label{sec:Proof_ConvUpperBound}
Consider a certain learning step $t$ in a $l$-th global interval, i.e., $t \in  [ \p{l-1}\uptau_1\uptau_2, l\uptau_1\uptau_2 )$. Recall that $\v_{\bkt{l}}\p{t+1} = \v_{\bkt{l}}\p{t} - \eta \nabla F\p{\v_{\bkt{l}}\p{t}}$ from Eq.~(6) in Section \ref{sec:Assumption}. Since $F$ is convex, an upper bound of the loss divergence between the virtual global model and the optimal global model can be expressed as Eq.~\eqref{eq:FFstar}.
\begin{equation}
\resizebox{0.9\columnwidth}{!}{$
\begin{multlined}
    F\p{\v_{\bkt{l}}\p{t}} - F\p{\w^*} \\
    \le \nabla F\p{\v_{\bkt{l}}\p{t}}^\top \p{\v_{\bkt{l}}\p{t} - \w^* } \hspace{3.5cm} \\
    = \frac{1}{\eta} \p{\v_{\bkt{l}}\p{t} - \v_{\bkt{l}}\p{t+1}}^\top \p{\v_{\bkt{l}}\p{t} - \w^* } \hspace{1.8cm}\\
    \begin{aligned}[t]= \frac{1}{2\eta} \p*{\norm{\v_{\bkt{l}}\p{t} - \v_{\bkt{l}}\p{t+1}}^2 \hspace{3.5cm}\\
    + \norm{ \v_{\bkt{l}}\p{t} - \w^*}^2 - \norm{\v_{\bkt{l}}\p{t+1} -\w^* }^2 } \end{aligned} \\
    % = \frac{1}{2\eta} \p{\norm{\v_{\bkt{l}}\p{t} - \v_{\bkt{l}}\p{t+1}}^2 + \norm{ \v_{\bkt{l}}\p{t} - \w^*}^2 - \norm{\v_{\bkt{l}}\p{t+1} -\w^* }^2 } \\
    = \frac{\eta}{2}\norm{\nabla F\p{\v_{\bkt{l}}\p{t}}}^2 \hspace{5cm} \\
    + \frac{1}{2\eta} \p{ \norm{ \v_{\bkt{l}}\p{t} - \w^*}^2 - \norm{\v_{\bkt{l}}\p{t+1} -\w^* }^2 } \label{eq:FFstar}
\end{multlined}
$}
\end{equation}  
Additionally, since $F$ is convex and $\beta$-smooth, when $\eta \le \frac{1}{\beta}$, one can derive Eq.~\eqref{eq:FFvv}.
\begin{equation}
\resizebox{0.8\columnwidth}{!}{$
\begin{multlined}
    F\p{\v_{\bkt{l}}\p{t}} - F\p{\v_{\bkt{l}}\p{t+1}} \\
    \ge \nabla F\p{\v_{\bkt{l}}\p{t}}^\top \p{\v_{\bkt{l}}\p{t} - \v_{\bkt{l}}\p{t+1}} \hspace{0.6cm} \\
    - \frac{\beta}{2}\norm{\v_{\bkt{l}}\p{t} - \v_{\bkt{l}}\p{t+1}}^2 \\
    = \eta \norm{\nabla F\p{\v_{\bkt{l}}\p{t}}}^2  - \frac{\beta \eta^2}{2}\norm{\nabla F\p{\v_{\bkt{l}}\p{t}}}^2 \\
    \ge \frac{\eta}{2}\norm{\nabla F\p{\v_{\bkt{l}}\p{t}}}^2 \hspace{3.7cm} \label{eq:FFvv}
\end{multlined}
$}
\end{equation}  
From Eq.~\eqref{eq:FFstar} and \eqref{eq:FFvv}, Eq.~\eqref{eq:FvFw} is derived by straightforward mathematics.
\begin{equation}
\resizebox{0.7\columnwidth}{!}{$
\begin{multlined}
    F\p{\v_{\bkt{l}}\p{n \uptau_1 \uptau_2 }} - F\p{\w^*} \\
    \begin{aligned}[t]\le \frac{1}{2\uptau_1 \uptau_2 \eta} \p*{ \norm{ \v_{\bkt{l}}\p{(l-1) \uptau_1 \uptau_2 } - \w^*}^2 \hspace{0.5cm} \\
    - \norm{\v_{\bkt{l}}\p{n \uptau_1 \uptau_2 } -\w^* }^2 } \end{aligned} \\
    % \le \frac{1}{2\uptau_1 \uptau_2 \eta} \p{ \norm{ \v_{\bkt{l}}\p{(l-1) \uptau_1 \uptau_2 } - \w^*}^2 - \norm{\v_{\bkt{l}}\p{n \uptau_1 \uptau_2 } -\w^* }^2 } \\
    \le \frac{1}{2\uptau_1 \uptau_2 \eta \omega} \hspace{4.4cm}\label{eq:FvFw} \\
\end{multlined}
$}
\end{equation}
From Eq.~\eqref{eq:FvFw} and Lemma~2, Theorem~1 can be proven as in Eq.~\eqref{eq:thm1pf}.
\begin{equation}
\label{eq:thm1pf}
\resizebox{\columnwidth}{!}{$
\begin{multlined}
    F\p{\w\p{T}} - F\p{\w^*} \le \frac{1}{2\uptau_1 \uptau_2 \eta \omega} \\
    + \rho \p{\frac{ \delta }{ \beta } \p{ \p{ \eta\beta + 1 }^{\uptau_1 } - 1 } + \frac{ \Delta }{ \beta } \p{ \p{ \eta\beta + 1 }^{ \uptau_1\uptau_2 } - 1 } }
\end{multlined}
$}
\end{equation}

\section{Advanced Implementation Technique}
\label{sec:Advanced}

\begin{figure}[t!]
    \centering
    \includegraphics[width=0.9\columnwidth]{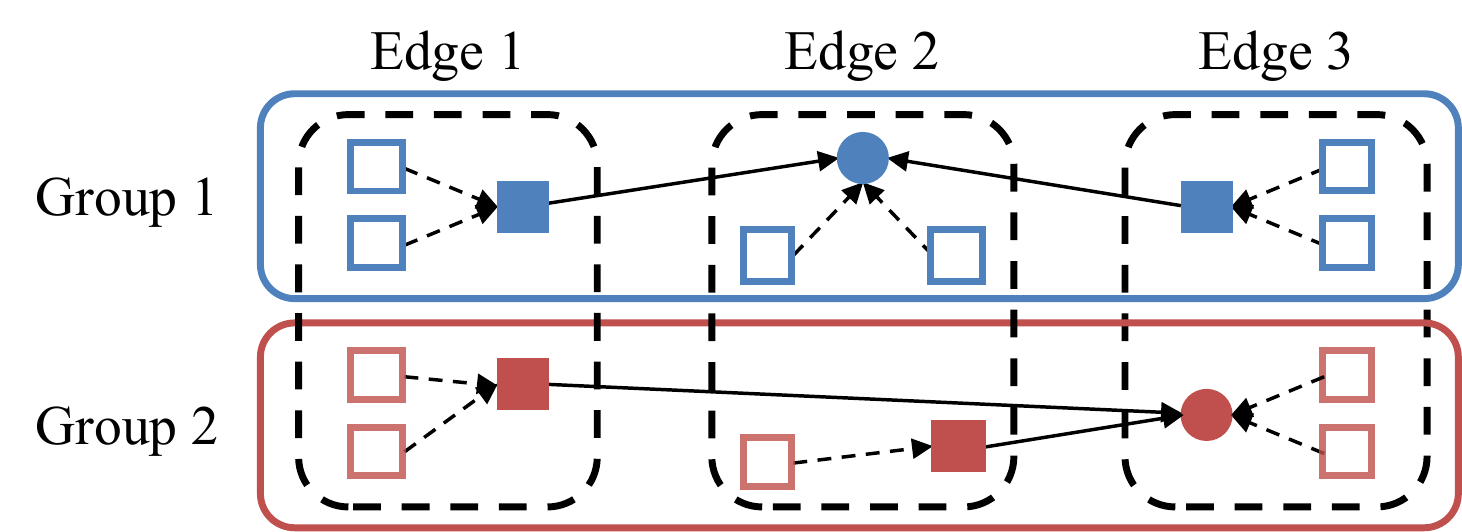}
    \caption{An example of combined aggregation.}
    \label{fig:CombinedAggr}
\end{figure} 

In addition, we propose a novel \emph{combined aggregation} technique that reduces the size of communication data in FedAvg-IC. \autoref{fig:CombinedAggr} represents an example of the combined aggregation. It is obvious that communication cost is almost negligible within an edge, and a group may consist of a few edges. Then, a certain local aggregation server\,(solid squares) can locally aggregate\,(dashed arrows) local models\,(hollow squares) in the same edge and send\,(solid arrows) the locally aggregated model to the group or global aggregation server\,(solid circles) with reduced communication data. This technique can be applied to both group and global aggregations as well as both group and global broadcasts inversely. We also note that this technique is similar to the partial aggregation of Dryad\,\citep{Dryad} and the combiner of MapReduce\,\citep{MapReduce}.

\section{Supplementary Evaluation Material}

\subsection{Experimental Setting Details}
\label{sec:ExperimentalSettingDetails}

{\bf \underline{Configuration}}:
We developed a federated learning simulator to extensively evaluate the performance of various algorithms, models, datasets, and networks. We used TensorFlow 1.14.0 to implement federated learning\footnote{TensorFlow Federated\,\citep{TFF} does not fully support the parallelism level of this simulation yet.} and ns-3 3.30 to simulate the network on servers with Intel Core i7-6700 and NVIDIA TITAN X. For reproducibility, we provide the source code at \url{https://bit.ly/39g10Ip}.
% \url{https://github.com/kaist-dmlab/FL-Sim}\footnote{All algorithms and models were implemented by us.}.

{\bf \underline{Models}}:
We used following \emph{three} training models.
\begin{itemize}[noitemsep,leftmargin=10pt,nosep]
\item
    The softmax regression\,(SR) involved 7,850 parameters.
\item
    The 2 layered perceptron neural network\,(2NN) contained two hidden layers each with 200 units and ReLU activiation; it contained 199,210 parameters.
\item
    The convolutional neural network\,(CNN) contained two 5$\times$5 convolutional layers with 64 channels, each followed by 2$\times$2 max pooling and local response normalization. After the two convolutional layers, a fully-connected layer with 256 units and ReLU activation was added; the output layer with softmax activation was added. The CNN contained 369,098 parameters. For CelebA, the benchmark CNN, provided by \citet{caldas2018leaf}, was used, and it contained 124,808 parameters.
\end{itemize}

{\bf \underline{Metrics}}:
We evaluated the performance of the algorithms using the following metrics.
The \textit{test accuracy} was measured to evaluate the training progress and the predictive accuracy, respectively. In addition, the \textit{epoch} and the \textit{time} taken to reach a target test accuracy were measured to evaluate the convergence speed.

{\bf \underline{Hyperparameters}}:
\autoref{table:Hyperparameters} lists the hyperparameters used for the model and the algorithm.

\begin{table}[t!]
\caption{Summary of parameters\,(the default value in bold).}
\centering
\small
\begin{tabular}{ccc} \toprule
    \textbf{Category} & \textbf{Paramter} & \textbf{Value} \\ \midrule
    \multirow{3}{*}{Model} & Batch size & 32, 64, \textbf{128}, 256, 512 \\
    & Learning rate & $10^{-3}$, $\cdots$, $\boldsymbol{10^{-1}}$, $\cdots$, $10^{3}$ \\
    & Learning rate decay & 0.99 \\ \midrule
    \multirow{2}{*}{Algorithm} & Learning steps & $\uptau=5$, $\uptau_1=1$, $\uptau_2=5$ \\
    & \# of groups & 2, \textbf{5}, 10, 15, 20, 30 \\ \bottomrule
    % \multirow{7}{*}{Environment} & \# of nodes & \textbf{100}, 150, 200, 250, 300 \\
    % & \# of edges & \textbf{10}, 15, 20, 25, 30 \\
    % & Data distribution & \textbf{Normal}, Exponential \\
    % & Network type & \textbf{Fat tree}, Jellyfish \\
    % & Processing speed & \textbf{5}, 250 GFLOPS \\
    % & Link speed & 10, \textbf{100} MBps \\
    % & Link latency & 1 ms \\ \bottomrule
\end{tabular}
\label{table:Hyperparameters}
\end{table}

\begin{itemize}[noitemsep,leftmargin=10pt,nosep]
\item
    \textbf{Model}: We searched the best batch size and learning rate for each model as follows. For the SR, 2NN and CNN models, the batch size was varied from $32$ to $512$ with an increment rate of $2$, and the learning rate\,($\eta$) was varied from $10^{-3}$ to $10^{3}$ with an increment rate of $10$. For the CNN model used for the CelebA dataset, the batch size was set to $5$, and the learning rate was set to $0.001$, as suggested by \citet{caldas2018leaf}.
\item
    \textbf{Algorithm}: FedAvg takes a single learning step\,($\uptau$) as its input, whereas HierFAVG and FedAvg-IC take two learning steps\,($\uptau_1$ and $\uptau_2$). To solely focus on the effects of communication, the product of all the learning steps of each algorithm was determined to be $5$, which was one of the suggested values by \citet{FedAvg}. In FedAvg-IC, the number of groups $\abs{\mathcal{K}}$ was varied in the range of $\bkt{2, 5, 10, 15, 20, 30}$\footnote{A sophisticated heuristic for determining the number of groups is an important issue in group federated learning, and we leave it as future work.}. We also note that, in Algorithm 1, the \textsc{k-Medoid\_Grouping} was performed only once at the beginning\,(Line 14), and the number of maximum steady steps was set to $1$\,(Line 19), which exhibited sufficiently high accuracy for the most of experiments despite the decreased optimization opportunities.
\item
    \textbf{Environment}: We determined the default environment parameters by adopting commonly used ones in the previous studies\,\citep{FedAvg,HierFAVG,AdaptiveFL,caldas2018leaf}.
\end{itemize}

In addition, to incorporate the unbalanced property\,\citep{FedAvg}, the number of classes per node or edge and the number of data examples per node were randomly sampled from the normal distribution.

\subsection{Evaluation Details}
\label{sec:EvaluationDetails}

\begin{figure}[t!]
\centering
    \begin{minipage}{.47\columnwidth}
        \includegraphics[width=\textwidth]{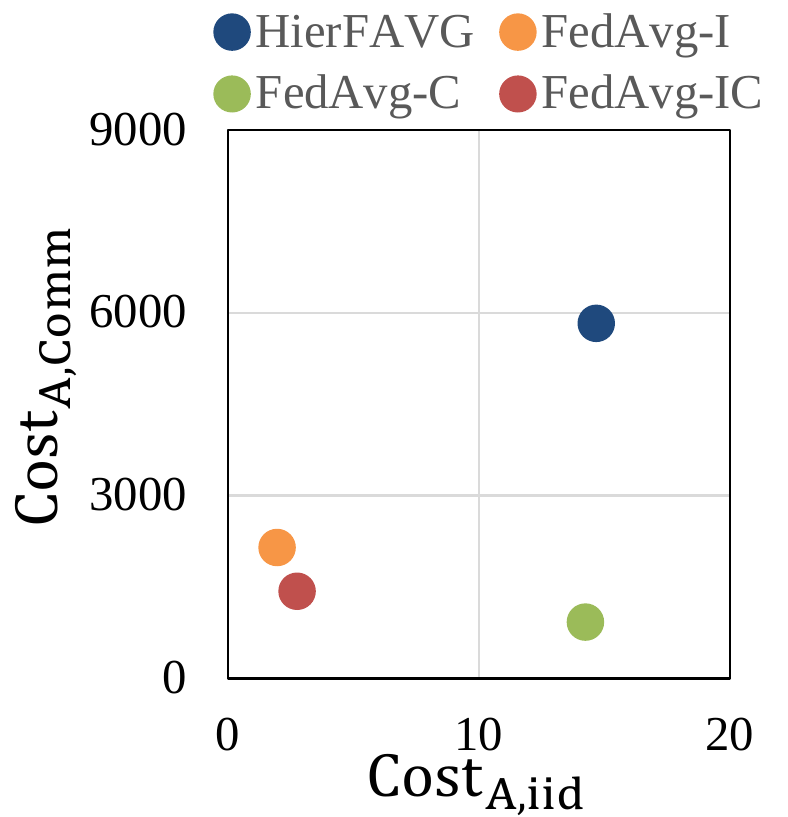}
        \caption{Cost analysis of the CNN on MNIST-F.}
        \label{fig:CostAnalysis}
    \end{minipage}
    \hspace{2pt}
    \begin{minipage}{.47\columnwidth}
        \includegraphics[width=\textwidth]{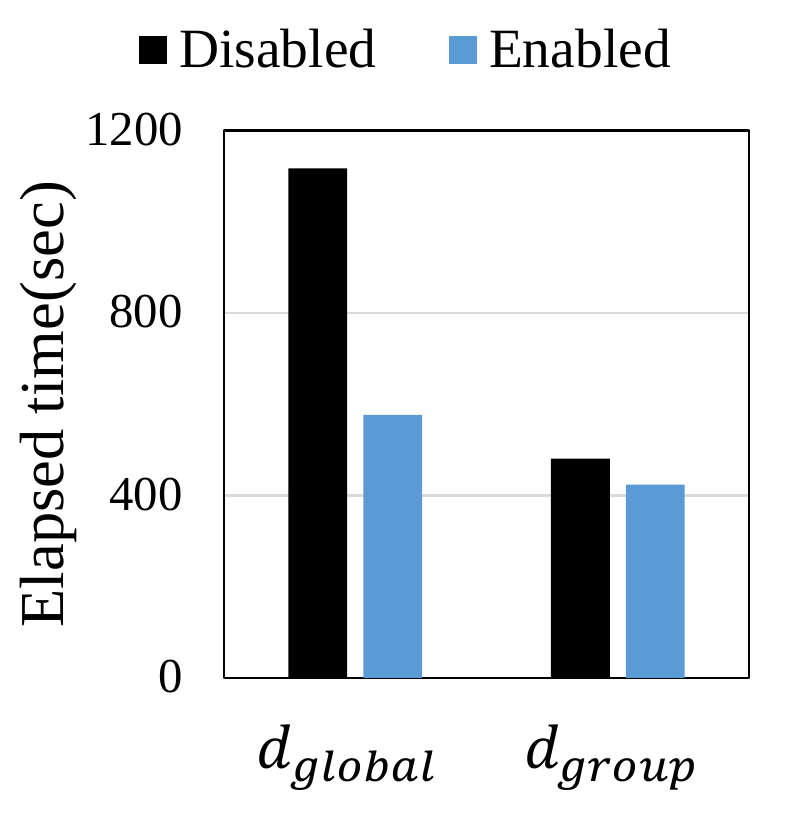}
        \caption{Effect of the combined aggregation.}
        \label{fig:EffectCombinedAggr}
    \end{minipage}
\end{figure}
The higher accuracy of FedAvg-IC in Section \ref{sec:AccuracyResults} is attributed to a decreased $\Delta$ in the IID cost in Eq.~\eqref{eq:Cost}. As shown in \autoref{fig:CostAnalysis} for the CNN on MNIST-F\,(\autoref{fig:accuracy-cnn-mnist_mnist-f}) with the \emph{Dtt} setting, FedAvg-IC significantly decreased the $\textsc{Cost}_{A,iid}$ that models $\Delta$. This conforms to Remark \ref{remark:Delta}.

The faster convergence speed of FedAvg-IC in Section \ref{sec:EfficiencyResults} is attributed to a decreased $d_{group}$ in Eq.~\eqref{eq:Cost} as well as a decreased communication data size by the combined aggregation. As shown in \autoref{fig:CostAnalysis}, FedAvg-IC significantly decreased the $\textsc{Cost}_{A,comm}$ that models $d_{group}$. This is consistent with Remark \ref{remark:Comm}. It should be also noted that FedAvg-IC finds a set of Pareto optimal solutions in \autoref{fig:CostAnalysis} whereas HierFAVG exhibits non-optimized costs. Furthermore, as shown in \autoref{fig:EffectCombinedAggr}, the communication time---especially, for the global aggregation---dropped rapidly when the combined aggregation was enabled.

\subsection{Additional Results}

\subsubsection{Effects of Different Simulation Settings}
\label{sec:EvalSimulation}

\urlstyle{rm}
{\bf \underline{Effects of Computation Settings}}:
We compared different processing speeds\,(5\footnote{This value is the average speed of Exynos 8895 in Samsung Galaxy S8. See \url{https://www.anandtech.com/show/11540/samsung-galaxy-s8-exynos-versus-snapdragon/2}.} and 250\footnote{This value is the average speed of PowerVR GT7600, the most widely-used smartphone GPU in 2019. See \url{https://deviceatlas.com/blog/most-used-smartphone-gpu}.} GFLOPS). As shown in \autoref{fig:EvalCommAndComp}, the results with a high processing speed\,(\autoref{fig:accuracy-cnn-mnist_mnist-o_link10_proc250} and \ref{fig:accuracy-cnn-mnist_mnist-o_link100_proc250}) exhibited higher accuracy than the ones with a low processing speed\,(\autoref{fig:accuracy-cnn-mnist_mnist-o_link10_proc5} and \ref{fig:accuracy-cnn-mnist_mnist-o_link100_proc5}), which is attributed to the increased number of epochs. Furthermore, we investigated the effects of different learning steps. As shown in \autoref{fig:LearningSteps}, the results with a small number of steps\,(\autoref{fig:product5}) exhibited higher accuracy than the ones with a large number of steps\,(\autoref{fig:product25}), because the convergence upper bound became larger with a  larger number of learning steps according to  Theorem~\ref{thm:ConvUpperBound}. It should be also noted that FedAvg-IC exhibited the lowest accuracy degradation, as indicated by the arrows in \autoref{fig:product25}; from Remark 1, when $\Delta$ is sufficiently minimized, other parameters hardly influence the convergence upper bound.

\begin{figure}[t!]
\captionsetup[subfigure]{justification=centering}
\centering
    \includegraphics[height=0.7cm]{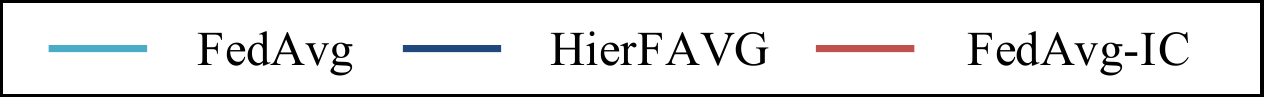} \\
    \includegraphics[width=\linewidth]{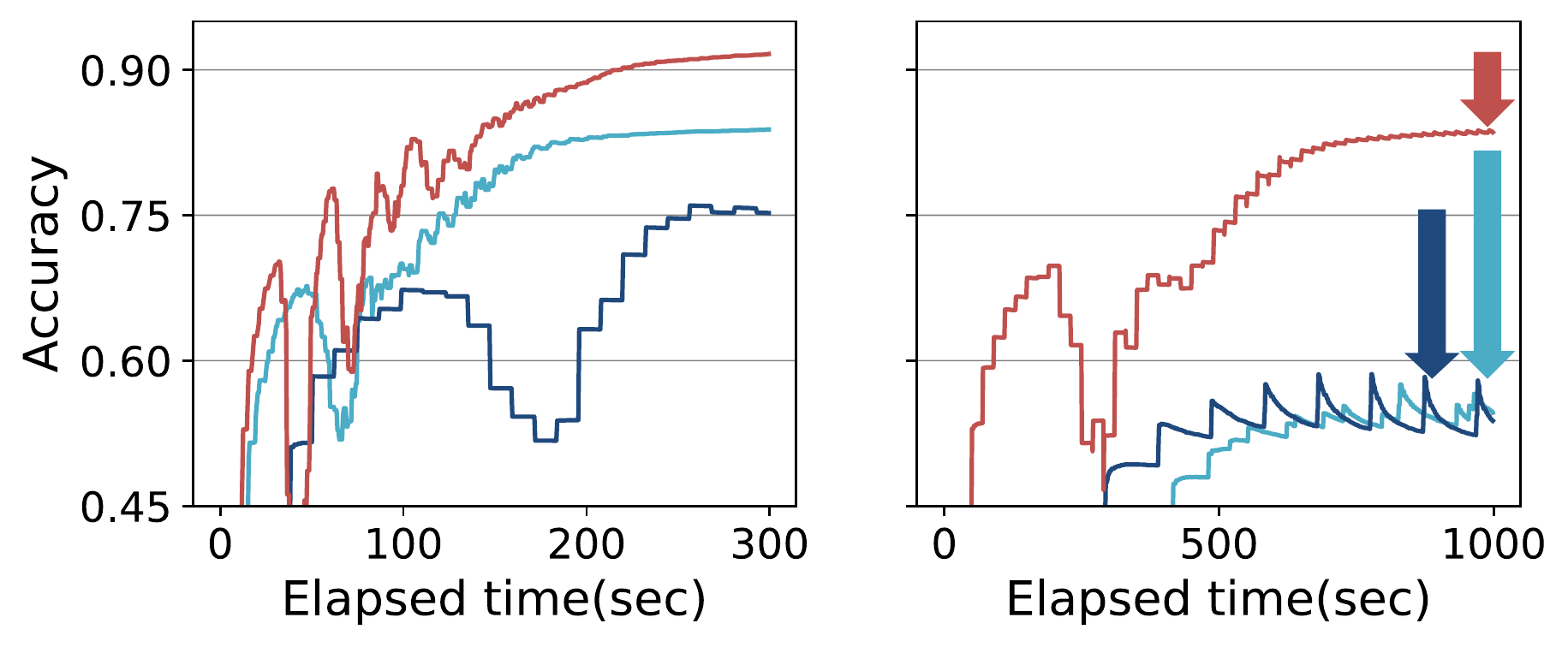} \\
    \begin{minipage}[t]{.50\linewidth}
    \centering
    \subcaption{A small number of steps\\($\uptau=5$ \& $\p{\uptau_1,\uptau_2}=\p{1,5}$).}\label{fig:product5}
    \end{minipage}
    \begin{minipage}[t]{.47\linewidth}
    \centering
    \subcaption{A large number of steps\\($\uptau=25$ \& $\p{\uptau_1,\uptau_2}=\p{5,5}$).}\label{fig:product25}
    \end{minipage}
\caption{Effects of learning steps for the CNN on MNIST-O\,(\emph{Dtt}).}
\label{fig:LearningSteps}
\end{figure}

{\bf \underline{Effects of Communication Settings}}: We compared different network types\,(fat tree and jellyfish) and link speeds\,(10 and 100 MBps\footnote{These values represent the state-of-the-art mobile connection speeds\,\citep{cisco}.}). As shown in \autoref{fig:EvalCommAndComp}, the results with the jellyfish\,(dashed lines) converged faster than the ones with the fat tree\,(solid lines), which is attributed to the increased throughput of the cost-efficient jellyfish network\,\citep{jellyfish}. In addition, the results with a high link speed\,(\autoref{fig:accuracy-cnn-mnist_mnist-o_link100_proc5} and \ref{fig:accuracy-cnn-mnist_mnist-o_link100_proc250}) converged faster than the ones with a low link speed\,(\autoref{fig:accuracy-cnn-mnist_mnist-o_link10_proc5} and \ref{fig:accuracy-cnn-mnist_mnist-o_link10_proc250}), as marked by small circles in the figures with a high link speed that represent the time taken to train the same number of epochs as with a low link speed.

\begin{figure*}[t!]
\captionsetup[subfigure]{justification=centering}
\centering
    \includegraphics[height=1.0cm]{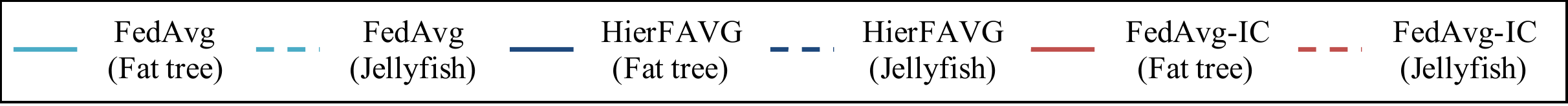} \\
    \includegraphics[width=\linewidth]{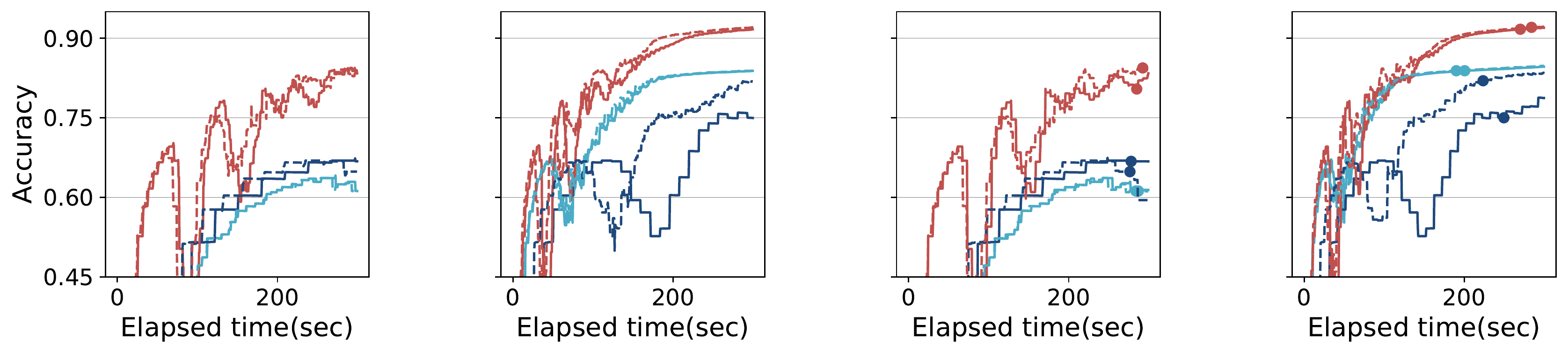} \\
    \begin{minipage}[t]{.30\linewidth}
    \centering
    \subcaption{Link speed=10\\Processing speed=5.}\label{fig:accuracy-cnn-mnist_mnist-o_link10_proc5}
    \end{minipage}
    \begin{minipage}[t]{.19\linewidth}
    \centering
    \subcaption{Link speed=10\\Processing speed=250.}\label{fig:accuracy-cnn-mnist_mnist-o_link10_proc250}
    \end{minipage}
    \begin{minipage}[t]{.31\linewidth}
    \centering
    \subcaption{Link speed=100\\Processing speed=5.}\label{fig:accuracy-cnn-mnist_mnist-o_link100_proc5}
    \end{minipage}
    \begin{minipage}[t]{.18\linewidth}
    \centering
    \subcaption{Link speed=100\\Processing speed=250.}\label{fig:accuracy-cnn-mnist_mnist-o_link100_proc250}
    \end{minipage}
\caption{Effects of different communication and computation settings for the CNN on MNIST-O\,(\emph{Dtt}). The standard deviation is not represented here to clearly convey the differences.}
\label{fig:EvalCommAndComp}
\end{figure*}

{\bf \underline{Effects of Data Distributions}}:
\autoref{fig:EvalDataDistribution} shows the accuracy results for different data distributions. As the variance of a distribution became higher\,(\autoref{fig:accuracy-cnn-mnist_mnist-o_normal-sd5} and \ref{fig:accuracy-cnn-mnist_mnist-o_normal-exp}), the learning curves for all algorithms became noisier. Nevertheless, FedAvg-IC still converged the fastest even with the noisier curves. We note that the mean of each data distribution does not need to be varied because, from the explanation of the data distribution in Section \ref{sec:ExperimentalSetting}, it was determined to be the number of classes or data examples per node.

% Here, we change the sampling distribution for the number of class and data in each node and edge by altering the distribution itself and its parameter. Thereby we can change the overall data distribution. We sample the number of class in each node and edge from a certain distribution with pre-assigned class diversity. For example, we sample the number of classes for an edge from a truncated normal distribution with the mean as a tenth of the number of total class in a dataset if the class diversity of an edge is assigned as \textbf{\textit{Dtt}}.

\begin{figure*}[t!]
\centering
    \includegraphics[height=0.7cm]{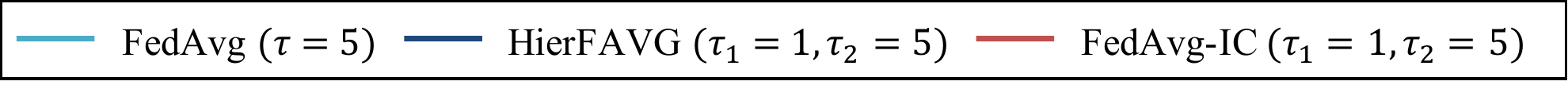} \\
    \includegraphics[width=\linewidth]{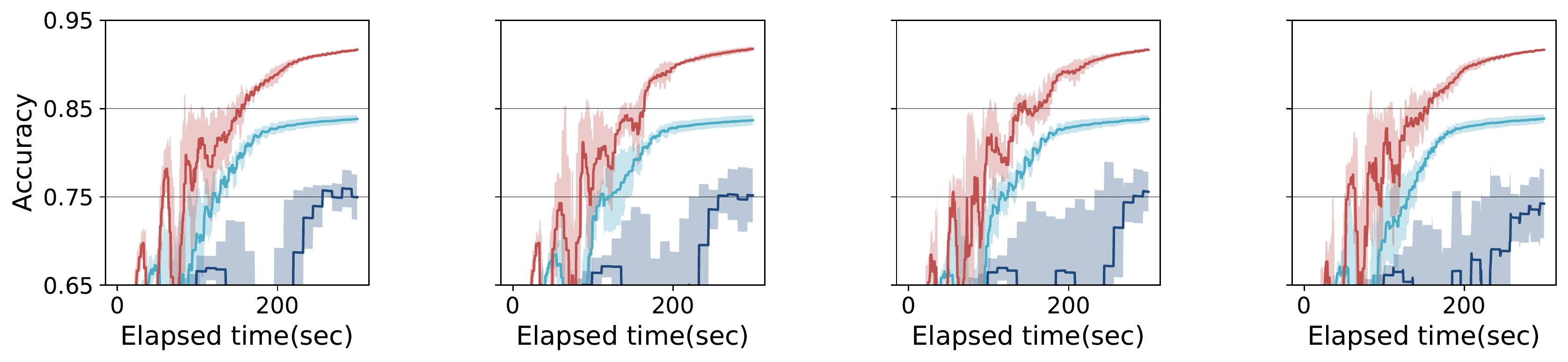} \\
    \begin{minipage}[t]{.30\linewidth}
    \centering
    \subcaption{Normal with SD=1.}\label{fig:accuracy-cnn-mnist_mnist-o_normal-sd1}
    \end{minipage}
    \begin{minipage}[t]{.20\linewidth}
    \centering
    \subcaption{Normal with SD=3.}\label{fig:accuracy-cnn-mnist_mnist-o_normal-sd3}
    \end{minipage}
    \begin{minipage}[t]{.29\linewidth}
    \centering
    \subcaption{Normal with SD=5.}\label{fig:accuracy-cnn-mnist_mnist-o_normal-sd5}
    \end{minipage}
    \begin{minipage}[t]{.19\linewidth}
    \centering
    \subcaption{Exponential.}\label{fig:accuracy-cnn-mnist_mnist-o_normal-exp}
    \end{minipage}
\caption{Effects of data distributions for the CNN on MNIST-O\,(\emph{Dtt}).}
\label{fig:EvalDataDistribution}
\end{figure*}

\subsubsection{Effects of the Degree of Non-IIDness}
\autoref{fig:ClassDiversity} shows the effects of class diversity\,(i.e., non-IIDness). In all cases, FedAvg-IC outperformed the state-of-the-art algorithms. When the data distribution in a node is non-IID\,(i.e., \emph{Dtt}, \emph{Dtq}, and \emph{Dth}), FedAvg that is a node-based learning did not work well. In contrast, when the data distribution in an edge is non-IID\,(i.e., \emph{Dtt}, \emph{Dtq}, and \emph{Dqq}), HierFAVG that is an edge-based learning did not work well. It should be noted that, because a node or an edge is very unlikely to be perfectly IID, the superiority of FedAvg-IC over the others will be valid in the real-world scenarios. Further evaluation with more realistic data distribution remains as future work.
% Notably, even HierFAVG did not work well with the \emph{Dtt} setting. However, because the other settings\,(i.e., \emph{Dth}, \emph{Dqh}, and \emph{Dhh}) are \emph{very} favorable to HierFAVG, the results of all algorithms were quite close to each other. As an edge is unlikely to be perfectly IID, the superiority of FedAvg-IC over the others will be valid in the real world. Further evaluation with more realistic data distribution remains as future work.}

\autoref{fig:ClassDiversity_mnist-f}, \autoref{fig:ClassDiversity_femnist}, and \autoref{fig:ClassDiversity_celeba} show the accuracy results on MNIST-F, FEMNIST, and CelebA, respectively. The overall trends are shown to be similar to \autoref{fig:ClassDiversity}. The learning curves are represented with the elapsed time in the subfigures (a)--(c) and with the number of epochs in the subfigures (d)--(f). The results with the \emph{Dqq}, \emph{Dqh}, and \emph{Dhh} settings are omitted because all curves closely overlap as before.

\begin{figure*}[t!]
\centering
    \includegraphics[height=0.7cm]{figure/legend2.pdf} \\
    \includegraphics[width=\linewidth]{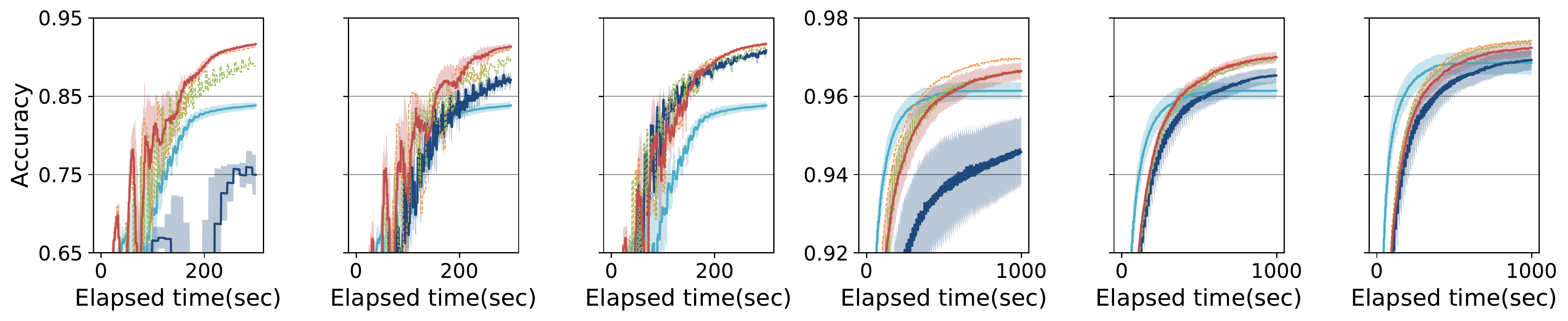} \\
    \begin{minipage}[t]{.20\linewidth}
    \centering
    \subcaption{Dtt.}\label{fig:accuracy-cnn-mnist_mnist-o-dtt}
    \end{minipage}
    \begin{minipage}[t]{.12\linewidth}
    \centering
    \subcaption{Dtq.}\label{fig:accuracy-cnn-mnist_mnist-o-dtq}
    \end{minipage}
    \begin{minipage}[t]{.19\linewidth}
    \centering
    \subcaption{Dth.}\label{fig:accuracy-cnn-mnist_mnist-o-dth}
    \end{minipage}
    \begin{minipage}[t]{.13\linewidth}
    \centering
    \subcaption{Dqq.}\label{fig:accuracy-cnn-mnist_mnist-o-dqq}
    \end{minipage}
    \begin{minipage}[t]{.18\linewidth}
    \centering
    \subcaption{Dqh.}\label{fig:accuracy-cnn-mnist_mnist-o-dqh}
    \end{minipage}
    \begin{minipage}[t]{.14\linewidth}
    \centering
    \subcaption{Dhh.}\label{fig:accuracy-cnn-mnist_mnist-o-dhh}
    \end{minipage}
\caption{Effects of class diversity for the CNN on MNIST-O (\autoref{fig:accuracy-cnn-mnist_mnist-o}): \emph{Dtt} Non-IID $\longleftrightarrow$ IID \emph{Dhh}.}
\label{fig:ClassDiversity}
\end{figure*}

\begin{figure*}[t!]
\centering
    \includegraphics[width=\linewidth]{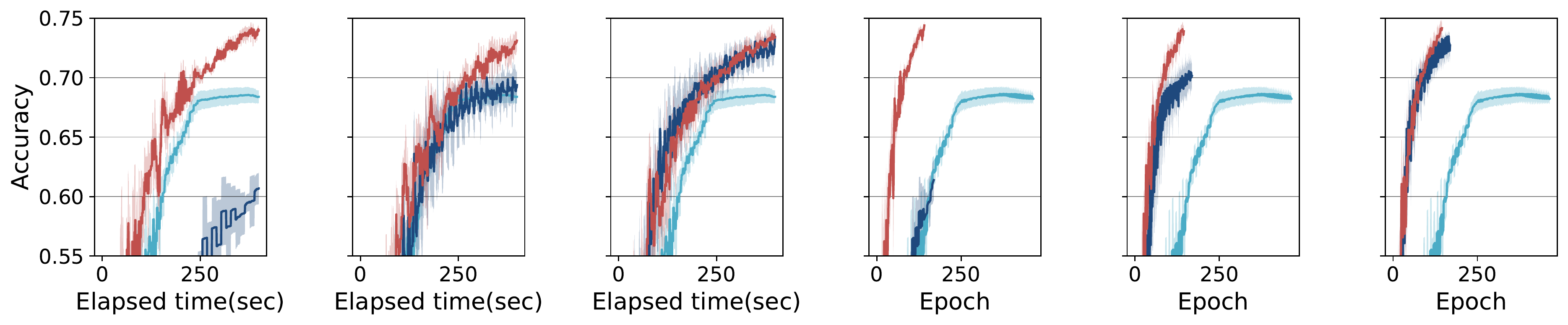} \\
    \begin{minipage}[t]{.21\linewidth}
    \centering
    \subcaption{Dtt-time.}\label{fig:accuracy-cnn-mnist_mnist-f-dtt-time}
    \end{minipage}
    \begin{minipage}[t]{.12\linewidth}
    \centering
    \subcaption{Dtq-time.}\label{fig:accuracy-cnn-mnist_mnist-f-dtq-time}
    \end{minipage}
    \begin{minipage}[t]{.20\linewidth}
    \centering
    \subcaption{Dth-time.}\label{fig:accuracy-cnn-mnist_mnist-f-dth-time}
    \end{minipage}
    \begin{minipage}[t]{.11\linewidth}
    \centering
    \subcaption{Dtt-epoch.}\label{fig:accuracy-cnn-mnist_mnist-f-dtt-epoch}
    \end{minipage}
    \begin{minipage}[t]{.21\linewidth}
    \centering
    \subcaption{Dtq-epoch.}\label{fig:accuracy-cnn-mnist_mnist-f-dtq-epoch}
    \end{minipage}
    \begin{minipage}[t]{.11\linewidth}
    \centering
    \subcaption{Dth-epoch.}\label{fig:accuracy-cnn-mnist_mnist-f-dth-epoch}
    \end{minipage}
\caption{Effects of class diversity for the CNN on MNIST-F: \emph{Dtt} Non-IID $\longleftrightarrow$ IID \emph{Dth}.}
\label{fig:ClassDiversity_mnist-f}
\end{figure*}

\begin{figure*}[t!]
\centering
    \includegraphics[width=\linewidth]{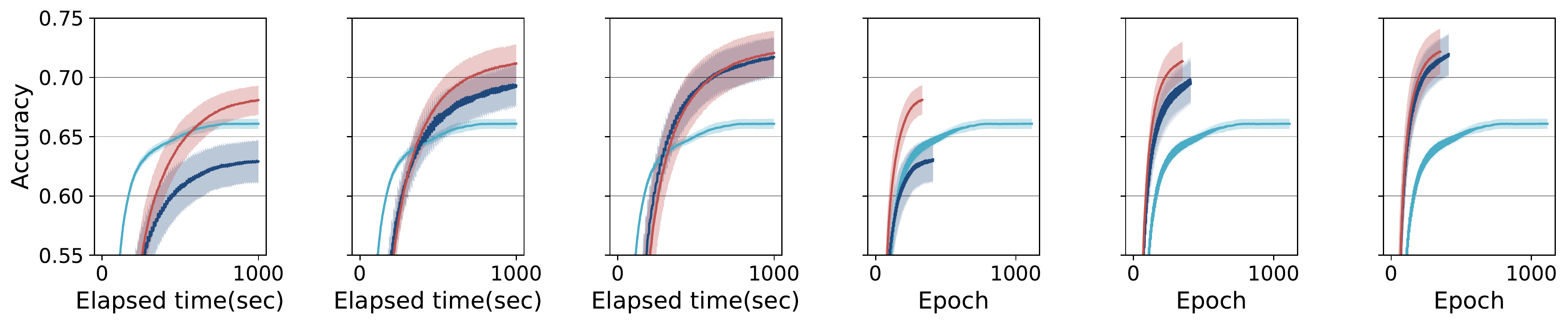} \\
    \begin{minipage}[t]{.21\linewidth}
    \centering
    \subcaption{Dtt-time.}\label{fig:accuracy-cnn-mnist_femnist-dtt-time}
    \end{minipage}
    \begin{minipage}[t]{.12\linewidth}
    \centering
    \subcaption{Dtq-time.}\label{fig:accuracy-cnn-mnist_femnist-dtq-time}
    \end{minipage}
    \begin{minipage}[t]{.20\linewidth}
    \centering
    \subcaption{Dth-time.}\label{fig:accuracy-cnn-mnist_femnist-dth-time}
    \end{minipage}
    \begin{minipage}[t]{.11\linewidth}
    \centering
    \subcaption{Dtt-epoch.}\label{fig:accuracy-cnn-mnist_femnist-dtt-epoch}
    \end{minipage}
    \begin{minipage}[t]{.21\linewidth}
    \centering
    \subcaption{Dtq-epoch.}\label{fig:accuracy-cnn-mnist_femnist-dtq-epoch}
    \end{minipage}
    \begin{minipage}[t]{.11\linewidth}
    \centering
    \subcaption{Dth-epoch.}\label{fig:accuracy-cnn-mnist_femnist-dth-epoch}
    \end{minipage}
\caption{Effects of class diversity for the CNN on FEMNIST: \emph{Dtt} Non-IID $\longleftrightarrow$ IID \emph{Dth}.}
\label{fig:ClassDiversity_femnist}
\end{figure*}

\begin{figure*}[t!]
\centering
    \includegraphics[width=\linewidth]{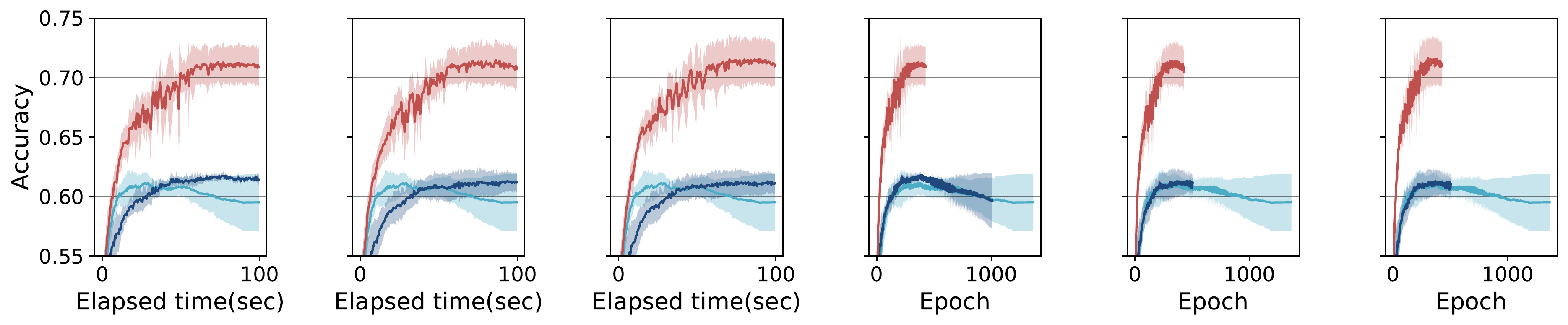} \\
    \begin{minipage}[t]{.21\linewidth}
    \centering
    \subcaption{Dtt-time.}\label{fig:accuracy-cnn-mnist_celeba-dtt-time}
    \end{minipage}
    \begin{minipage}[t]{.12\linewidth}
    \centering
    \subcaption{Dtq-time.}\label{fig:accuracy-cnn-mnist_celeba-dtq-time}
    \end{minipage}
    \begin{minipage}[t]{.20\linewidth}
    \centering
    \subcaption{Dth-time.}\label{fig:accuracy-cnn-mnist_celeba-dth-time}
    \end{minipage}
    \begin{minipage}[t]{.11\linewidth}
    \centering
    \subcaption{Dtt-epoch.}\label{fig:accuracy-cnn-mnist_celeba-dtt-epoch}
    \end{minipage}
    \begin{minipage}[t]{.21\linewidth}
    \centering
    \subcaption{Dtq-epoch.}\label{fig:accuracy-cnn-mnist_celeba-dtq-epoch}
    \end{minipage}
    \begin{minipage}[t]{.11\linewidth}
    \centering
    \subcaption{Dth-epoch.}\label{fig:accuracy-cnn-mnist_celeba-dth-epoch}
    \end{minipage}
\caption{Effects of class diversity for the CNN on CelebA: \emph{Dtt} Non-IID $\longleftrightarrow$ IID \emph{Dth}.}
\label{fig:ClassDiversity_celeba}
\end{figure*}

\end{document}